\documentclass[11pt]{article}

\usepackage[utf8]{inputenc} 
\usepackage[T1]{fontenc} 
\usepackage{amsthm}
\usepackage{mathtools}
\usepackage{graphicx} 
\usepackage{array} 

\usepackage{amsmath, amssymb, amsfonts, verbatim}
\usepackage{hyphenat,epsfig,subcaption,multirow}
\usepackage{nicefrac}
\usepackage{paralist}
\usepackage{enumitem}
\usepackage{blindtext}
\usepackage{cite}

\usepackage{caption}

\usepackage{xcolor}

\usepackage{tcolorbox}
\tcbuselibrary{skins,breakable}
\tcbset{enhanced jigsaw}

\usepackage[normalem]{ulem}
\usepackage[compact]{titlesec}

\definecolor{DarkRed}{rgb}{0.5,0.1,0.1}
\definecolor{DarkBlue}{rgb}{0.1,0.1,0.5}
\definecolor{RURed}{rgb}{0.8,0.1,0.1}

\usepackage{nameref}
\definecolor{ForestGreen}{rgb}{0.1333,0.5451,0.1333}
\definecolor{Red}{rgb}{0.9,0,0}
\usepackage[linktocpage=true,
	pagebackref=true,colorlinks,
	linkcolor=DarkRed,citecolor=ForestGreen,
	bookmarks,bookmarksopen,bookmarksnumbered]
	{hyperref}
\usepackage[noabbrev,nameinlink]{cleveref}
\crefname{property}{property}{Property}
\creflabelformat{property}{(#1)#2#3}
\crefname{equation}{eq}{Eq}
\creflabelformat{equation}{(#1)#2#3}

\usepackage{bm}
\usepackage{url}
\usepackage{xspace}
\usepackage[mathscr]{euscript}

\usepackage{tikz}
\usetikzlibrary{arrows}
\usetikzlibrary{arrows.meta}
\usetikzlibrary{shapes}
\usetikzlibrary{backgrounds}
\usetikzlibrary{positioning}
\usetikzlibrary{decorations.markings}
\usetikzlibrary{patterns}
\usetikzlibrary{calc}
\usetikzlibrary{fit}
\usetikzlibrary{snakes}
\tikzset{vertex/.style={circle, black, fill=Yellow, line width=1pt, draw, minimum width=8pt, minimum height=8pt, inner sep=0pt}}
	
\usepackage{mdframed}

\usepackage[noend]{algpseudocode}
\makeatletter
\def\BState{\State\hskip-\ALG@thistlm}
\makeatother

\usepackage{enumitem}
\usepackage{todonotes}

\usepackage[margin=1in]{geometry}

\newtheorem{theorem}{Theorem}
\newtheorem{lemma}{Lemma}[section]
\newtheorem{proposition}[lemma]{Proposition}
\newtheorem{corollary}[lemma]{Corollary}
\newtheorem{claim}[lemma]{Claim}
\newtheorem{fact}[lemma]{Fact}

\newtheorem{definition}[lemma]{Definition}

\newtheorem*{claim*}{Claim}
\newtheorem*{theorem*}{Theorem}
\newtheorem*{proposition*}{Proposition}
\newtheorem*{lemma*}{Lemma}
\newtheorem*{problem*}{Problem}

\crefname{lemma}{Lemma}{Lemmas}
\crefname{claim}{Claim}{Claims}

\newtheorem{mdresult}{Result}

\newtheorem{remark}[lemma]{Remark}
\newtheorem*{remark*}{Remark}
\newtheorem{observation}[lemma]{Observation}

\newtheoremstyle{restate}{}{}{\itshape}{}{\bfseries}{~(restated).}{.5em}{\thmnote{#3}}
\theoremstyle{restate}

\theoremstyle{definition}
\newtheorem{mdalg}{algorithm}

\newtheorem{mddist}{Distribution}

\allowdisplaybreaks

\renewcommand{\qed}{\nobreak \ifvmode \relax \else
      \ifdim\lastskip<1.5em \hskip-\lastskip
      \hskip1.5em plus0em minus0.5em \fi \nobreak
      \vrule height0.75em width0.5em depth0.25em\fi}

\newcommand{\Qed}[1]{\rlap{\qed$_{\textnormal{~~~\Cref{#1}}}$}}

\newcommand{\tvd}[2]{\ensuremath{\norm{#1 - #2}_{\mathrm{tvd}}}}

\newcommand{\eps}{\ensuremath{\varepsilon}}

\newcommand{\bracket}[1]{\left[#1\right]}
\newcommand{\paren}[1]{\ensuremath{\left(#1\right)}\xspace}
\newcommand{\card}[1]{\left\vert{#1}\right\vert}

\newcommand{\norm}[1]{\ensuremath{\|#1\|}}

\newcommand{\expect}[1]{\Exp\bracket{#1}}
\newcommand{\exprand}[2]{\Exp_{#1}\bracket{#2}}

\newcommand{\set}[1]{\ensuremath{\left\{ #1 \right\}}}
\newcommand{\poly}{\mbox{\rm poly}}
\newcommand{\polylog}[1]{\textnormal{polylog}\,(#1)\xspace}

\newcommand{\ALG}{\ensuremath{\mbox{\sc alg}}\xspace}

\DeclareMathOperator*{\Exp}{\ensuremath{{\mathbb{E}}}}
\DeclareMathOperator*{\Prob}{\ensuremath{\textnormal{Pr}}}
\renewcommand{\Pr}{\Prob}

\newenvironment{tbox}{\begin{tcolorbox}[
		enlarge top by=5pt,
		enlarge bottom by=5pt,
		 breakable,
		 boxsep=2pt,
                  left=5pt,
                  right=7pt,
                  top=10pt,
                  arc=0pt,
                  boxrule=1pt,toprule=1pt,
                  colback=white
                  ]
	}
{\end{tcolorbox}}

\newcommand{\event}{\ensuremath{\mathcal{E}}}

\newcommand{\kl}[2]{\ensuremath{\mathbb{D}(#1~||~#2)}}
\newcommand{\II}{\ensuremath{\mathbb{I}}}
\newcommand{\HH}{\ensuremath{\mathbb{H}}}
\newcommand{\mi}[2]{\ensuremath{\def\mione{#1}\def\mitwo{#2}\mireal}}
\newcommand{\mireal}[1][]{
  \ifx\relax#1\relax%
    \II(\mione \,; \mitwo)%
  \else%
    \II(\mione \,; \mitwo\mid #1)%
  \fi
}
\newcommand{\en}[1]{\ensuremath{\HH(#1)}}

\newcommand{\cD}{\mathcal{D}}

\newcommand{\bM}{\mathbf{M}\xspace}

\newcommand{\cO}{\ensuremath{\mathcal{O}}\xspace}
\newcommand{\cU}{\ensuremath{\mathcal{U}}\xspace}
\newcommand{\cS}{\ensuremath{\mathcal{S}}\xspace}

\newcommand{\sM}{\ensuremath{\mathsf{M}}\xspace}
\newcommand{\sPi}{\ensuremath{\mathsf{\Pi}}\xspace}
\newcommand{\sB}[1]{\ensuremath{\mathsf{B}_{#1}}\xspace}

\newcommand{\Picap}[1]{\ensuremath{\Pi_{\cap B_{#1}}}\xspace}
\newcommand{\sPicap}[1]{\ensuremath{\sPi_{\cap B_{#1}}}\xspace}

\newcommand{\ja}{\ensuremath{j}_{\text{arrive}}\xspace}

\newcommand{\dist}{\ensuremath{\mathcal{D}}}

\newcommand{\arm}{\ensuremath{\textnormal{arm}}\xspace}
\newcommand{\armstar}{\ensuremath{\arm^{*}}\xspace}
\newcommand{\armtilde}{\ensuremath{\widetilde{\arm}}\xspace}

\newcommand{\ptilde}{\ensuremath{\tilde{p}}\xspace}

\newcommand{\bern}[1]{\ensuremath{\textnormal{Bern}(#1)}\xspace}
\newcommand{\Thetatild}{\ensuremath{\widetilde{\Theta}}\xspace}

\newcommand{\distmu}{\ensuremath{\mathcal{U}}\xspace}

\newcommand{\batchind}[1]{\ensuremath{I(B_{#1})}\xspace}

\newcommand{\sample}[1]{\ensuremath{\textnormal{\textsf{Smp}}(#1)}}
\newcommand{\memory}[1]{\ensuremath{\textnormal{\textsf{Mem}}(#1)}}

\newcommand{\sampleb}[1]{\ensuremath{\textnormal{\textsf{Smp}}(\ALG)_{B_{P+1}:B_{#1+1}}}}

\renewcommand{\leq}{\leqslant}

\renewcommand{\geq}{\geqslant}

\newcommand{\myqed}[1]{\let\qed\relax \hspace*{\fill} #1 \ensuremath{\square}}

\title{The Best Arm Evades: Near-optimal Multi-pass Streaming Lower Bounds for Pure Exploration in Multi-armed Bandits}
 \author{Sepehr Assadi\footnote{(sepehr@assadi.info) Cheriton School of Computer Science, University of Waterloo. This work was done when the author was at Rutgers University and was supported in part by an NSF CAREER Grant CCF-2047061, a gift from Google Research, and a Rutgers University Fulcrum Award.} \and
 Chen Wang\footnote{(cw200@rice.edu) Department of Computer Science, Rice University, and Department of Computer Science \& Engineering, Texas A\&M University. This work was done when the author was at Rutgers University and was supported in part by an NSF CAREER Grant CCF-2047061, a gift from Google Research, and a Rutgers University Fulcrum Award.}
}

\date{}

\newenvironment{myquote}[1]%
  {\list{}{\leftmargin=#1\rightmargin=#1}\item[]}%
  {\endlist}

\begin{document}

\maketitle

\pagenumbering{roman}

\begin{abstract}
We give a near-optimal sample-pass trade-off for pure exploration in multi-armed bandits (MABs) via multi-pass streaming algorithms:
any streaming algorithm with sublinear memory that uses the optimal sample complexity of $O(n/\Delta^2)$ requires $\Omega(\log{(1/\Delta)}/\log\log{(1/\Delta)})$ passes. 
Here, $n$ is the number of arms and $\Delta$ is the reward gap between the best and the second-best arms. Our result matches the $O(\log(1/\Delta))$ pass algorithm of Jin et al. [ICML'21] (up to lower order terms) that only uses $O(1)$ memory and answers an open question posed by Assadi and Wang [STOC'20].
\end{abstract}

\clearpage

\setcounter{tocdepth}{3}
\tableofcontents

\clearpage

\pagenumbering{arabic}
\setcounter{page}{1}


\section{Introduction}
\label{sec:intro}
Pure exploration in multi-armed bandits (MABs) is a fundamental problem in machine learning (ML) and theoretical computer science (TCS). The classical setting of the problem is as follows: we are given $n$ arms with unknown sub-Gaussian reward distributions, and we want to find the best arm, defined as the arm with the highest expected reward, with a high probability and a small number of arm pulls. The problem has been extensively studied in the learning theory community (e.g. \cite{Even-Dar+02,MannorTs04,KalyanakrishnanSt10,AgarwalAAK17,ChenLQ17}), and it has found applications in various fields like online advertisement \cite{BertsimasM07,SchwartzBF17}, clinical trials \cite{villar2015multi,aziz2021multi}, content optimization \cite{AgarwalCE09,LiCLS10}, among others.

Under the classical (RAM) setting, the sample complexity of $\Theta(\frac{n}{\Delta^2})$ is shown to be necessary and sufficient to find the best arm with high constant probability (\!\!\!\cite{Even-Dar+02,Even-Dar+06,MannorTs04}, cf. \cite{KarninKS13,JamiesonMNB14}). Here, and throughout, $\Delta$ is used to denote the gap between the mean of the best and the second-best arms. On the flip side, all the classical algorithms require the entire set of arms to be stored for repeated visit. In modern large-scale applications, such a memory requirement may render the algorithms inefficient. In light of this, \cite{AssadiW20} introduced the \emph{streaming} multi-armed bandits model, in which the arms arrive one-by-one in a stream, and the algorithm is only allowed to store $o(n)$ arms, and ideally much smaller, like $O(1)$ or $\polylog{n}$ arms, at any time. Perhaps surprisingly, \cite{AssadiW20} showed that if we are given the value of $\Delta$ \emph{a priori}, there exists a single-pass streaming algorithm that finds the best arm with high constant probability, uses $O(\frac{n}{\Delta^2})$ samples, and only maintains a memory of a single extra arm.

The results of \cite{AssadiW20} led to a nascent line of work on MABs in the streaming model\cite{MaitiPK21,JinH0X21,AWneurips22,AgarwalKP22,Wang23Regret,LZWL23}. For the pure exploration problem, it has been shown by \cite{AWneurips22} that the prior knowledge of $\Delta$ is necessary for the single-pass algorithm: if this piece of information is not available and only the stream of arms itself is provided, then any single-pass algorithm with $o(n)$-arm memory that finds the best arm with high constant probability can incur an \emph{unbounded} sample complexity (as a function of $\Delta$). On the positive side, \cite{JinH0X21} designs an algorithm with $O(\frac{n}{\Delta^2})$ sample complexity and a memory of a single arm in $\log(\frac{1}{\Delta})$ passes, even if the knowledge of $\Delta$ is \emph{not} given a priori \footnote{The algorithm actually achieves a stronger instance-sensitive sample complexity -- see \Cref{sec:relate-work} for a discussion.}. This large gap between the positive results with $\log(\frac{1}{\Delta})$ passes and the lower bound in the single-pass setting presents us with the exciting open question:
\begin{myquote}{5pt}
\centering
\emph{If no additional knowledge is given apart from the stream, 
how many passes are \textbf{necessary} for streaming MABs algorithms with $o(n)$-arm memory to find the best arm with $O(\frac{n}{\Delta^2})$ samples?}
\end{myquote}
\vspace{-4pt}
The open question was initially mentioned by \cite{AssadiW20} and was later re-formulated in \cite{AWneurips22}\footnote{The problem is discussed at multiple open problem sessions of conferences and workshops, e.g, \href{https://waldo2021.github.io/}{WALD(O) 2021}.}. The quest of the tight pass bound is similar-in-spirit to the lower bounds in collaborative learning \cite{Tao0019} and multi-pass \emph{regret minimization} MABs \cite{AgarwalKP22}; however, none of the techniques in the aforementioned lower bounds can be directly used for this problem (see \Cref{subsec:technique} for a detailed discussion), which renders the open problem even more fascinating. 

In this work, we provide the answer to the open question: we show that (almost) $\Omega(\log(1/\Delta))$ passes are \emph{necessary} (up to exponentially smaller factors) for any algorithms with $o(n)$ memory to find the best arm. More formally, our main result can be presented as follows.
\begin{mdresult}[Informal version of \Cref{cor:main}]
\label{rst:main-result}
Any streaming algorithm that finds the best arm with probability at least $\frac{1999}{2000}$ using a memory of $o\paren{\frac{n}{\log^3(1/\Delta)}}$ arms and a sample complexity of $C \cdot \frac{n}{\Delta^2}$ for any constant $C$ has to use $\Omega\paren{\frac{\log(\frac{1}{\Delta})}{\log\log(\frac{1}{\Delta})}}$ passes.
\end{mdresult}
Our lower bound asymptotically matches the pass bound of the algorithm in \cite{JinH0X21} up to the exponentially smaller $O(\log\log(1/\Delta))$ term. Furthermore, as long as $\Delta \geq 1/2^{n^{1/3-\Omega(1)}}$, which is a quite natural assumption, our result demonstrates a sharp dichotomy on the pass-memory trade-off: if we use slightly less than $O(\log(1/\Delta))$ passes, no algorithm with even slightly less than $n$-arm memory can succeed with a good probability; however, if we slightly increase the number of passes to $O(\log(1/\Delta))$, it is possible to find the best arm with high constant probability with only a single arm of memory. 

\subsection{Our Techniques}
\label{subsec:technique}
 The proof of our result is based on a novel inductive argument that explicitly keeps track of the information revealed to the algorithm in each pass. This is in sharp contrast with all other lower bounds that address `rounds' or `passes' 
for MABs in similar contexts (e.g.,~\cite{AgarwalAAK17,Tao0019,Karpov0020,AgarwalKP22,KarpovZhangAAAI23}) that are based on \emph{round/pass elimination} ideas. 
To elaborate, let us take a closer look at \cite{AgarwalKP22}, which studied multi-pass streaming lower bounds in MABs for regret minimization. 
Roughly speaking, both \cite{AgarwalKP22} and our lower bound instances divide the stream into equal-sized \emph{batches}. 
Each batch contains a single arm with mean reward either $\frac{1}{2}$ or $>\frac{1}{2}$, and the rest of the arms in the batch have mean reward $\frac{1}{2}$.
The \emph{intuition} here is that by arranging the batches that \emph{may} have higher mean rewards to arrive later, the algorithm is forced to be `conservative' at each pass to only `eliminate' the last batch.
To this end, the main technical step of \cite{AgarwalKP22} is to reduce proving the lower bound for 
$P$-pass algorithms to proving a lower bound for $(P-1)$-pass algorithms---this is the so-called round/pass elimination idea.  
However, as the algorithm can gain information in each pass, the instance distribution from the algorithm's internal view is inevitably `more biased'. 
As such, a key part of the analysis in \cite{AgarwalKP22} is a delicate handling of the change in the distribution of instances from one round to the next, and ensures the change is not too much.

For our purpose, round elimination seems to ask too much from the argument to make sure that the distribution only slightly changes. As such, we proceed differently by \emph{allowing} the instance distribution to significantly change between rounds. 
Concretely, for a $P$-pass algorithm, we divide the arms into $(P+1)$ equal-sized batches, and arrange them in the \emph{reversed} order of the stream arrival, i.e., the stream is composed of $(B_{P+1}, B_{P},\cdots, B_{1})$. Each batch \emph{may} contain an arm with mean reward $\frac{1}{2}+\eta_{p}$, and the rest of the arms are `flat', i.e., with mean reward $\frac{1}{2}$. The parameter $\eta_{p}$ decreases by a polynomial factor of $1/P$, i.e., $\eta_{p+1}\leq ({1}/{P^{15}})\cdot \eta_{p}$. 
At each pass $p$, we show that the algorithm so far has not gained enough ``knowledge'' about the batch $B_p$ such that even if the algorithm knows that none of the batches $B_1,\ldots,B_{p-1}$ contain any arm with mean reward more than $\frac{1}{2}$,
it  still cannot decide whether $B_p$ has such an arm or not. 
This means that {if} the algorithm uses too many samples in the first $p$ passes, it risks breaking the guarantee on the sample complexity (if $B_p$ turns out to have a high reward arm), 
and otherwise if it does not make enough samples, it will not gain enough ``knowledge'' for batch $B_{p+1}$ and the subsequent pass. 

What has changed in this argument compared to prior approaches, say, in~\cite{AgarwalKP22}, is on how we interpreted this gain of knowledge: for us,
it is quite likely that the distributions of the batches change dramatically from the original distribution after each pass; we instead \emph{explicitly} account for the ability of the algorithm in $(1)$ determining whether a batch contains a high reward arm, and $(2)$ 
storing any high reward arm inside its memory. We shall track the probability of these events throughout the passes, sometimes even `revealing' extra information to the algorithm that are `not interesting', and use them inductively to establish
our lower bound. This approach may be of independent interest in other settings as well that target proving multi-pass/round lower bounds on sample-space tradeoffs for learning problems. 


Apart from the novel inductive argument, our techniques are distinct from \cite{AgarwalKP22} on two other aspects. First, in \cite{AgarwalKP22}, each batch may contain the arm with reward $>\frac{1}{2}$ with \emph{constant} probability. For the pure exploration problem, this means the best arm is among the last $\log{(P)}$ batches with very high probability, which makes the instance not hard. In contrast, our construction only uses $O(1/P)$ probability for each batch to have an arm that is `not flat' . 
Second, the techniques in \cite{AgarwalKP22} do \emph{not} factor in the dependence on number of arms $n$ (namely, their bounds only hold for fixed values of $n$); we extend a novel `arm-trapping' tool developed by \cite{AWneurips22} to remedy this.




\subsection{Related Work}
\label{sec:relate-work}
Apart from the $O(\frac{n}{\Delta^2})$ worst-case sample complexity, pure exploration in multi-armed bandits are also studied from the lens of the \emph{instance-sensitive} sample complexity, i.e. the bound as a function of $\{\Delta_{[i]}\}_{i=2}^{n}$, which are the mean reward gaps between the best and the $i$-th best arms. On this front, \cite{KarninKS13,JamiesonMNB14} devised algorithms that achieve $O(H_{2} := \sum_{i=2}^{n}\frac{1}{\Delta^2_{[i]}}\log\log(\frac{1}{\Delta_{[i]}}))$ sample complexity, which is almost optimal up to the doubly-logarithmic term. In the streaming setting, \cite{AWneurips22} showed that it is impossible for any algorithm with $o(n)$ memory to get the $O(H_{2})$ sample complexity without strong extra conditions; on the other hand, the algorithm in \cite{JinH0X21} achieves the $O(H_{2})$ sample complexity in $O(\log(1/\Delta))$ passes. We note that our lower bound naturally works in the instance-sensitive setting; as such, the sharp pass-memory trade-off also applies with this sample complexity.

In addition to pure exploration, streaming MABs are studied under the context of \emph{$\eps$-best arm identification} and \emph{regret minimization}. The $\eps$-best arm identification problem aims to find an arm whose reward is at most $\eps$-far from the best arm. On this front, the line of work by \cite{AssadiW20,MaitiPK21,JinH0X21} give algorithms that finds an $\eps$-best arm with $O(\frac{n}{\eps^2})$ samples and a single arm memory. For the regret minimization task, early work of \cite{LiauSPY18,ChaudhuriK20} gives multi-pass streaming algorithms, and \cite{MaitiPK21,Wang23Regret} provided single-pass tight single-pass upper and lower regret bounds. For multi-pass scenario, \cite{AgarwalKP22} provides a sharp memory-regret trade-off for multi-pass streaming MABs, and their construction shares a certain degree of similarity with ours. As we have discussed in \Cref{subsec:technique}, our techniques are substaintially different from theirs.
Very recently, \cite{HYZ24} further tightened the memory-regret trade-off for multi-pass algorithms using a different approach.

Aimed at modern massive data processing, MABs are also studied under other sublinear models. 
For instance, the settings of MABs under \emph{collaborative learning}, in which the sampling is done by multiple agents in parallel and the goal is to minimize the rounds of communications, has been extensively studied \cite{Tao0019,Karpov0020,KarpovZhangAAAI23}. We remark that the round lower bound in \cite{Tao0019} does \emph{not} imply a lower bound in our setting: the model requires simultaneous communication and cannot be simulated by streaming algorithms efficiently. 
The streaming expert advice problem studied by \cite{SrinivasWXZ22,PengZ23,ACNS23} is also closely related to the streaming MABs. There, the memory complexity is defined with the classical notion of bits, which is different from the memory constraint of our model. As such, the results between the two models are not directly comparable. 



\section{Preliminaries}
\label{sec:prelim}

\paragraph{Notation.} We frequently use random variables and their realizations in this paper. As a general rule, apart from a handful of self-contained proofs of technical lemmas, we use the \emph{sans serif} fonts (e.g., $\sM$) to denote the random variable and the normal font (e.g., $M$) to denote the realization. Throughout, we use $n$ to denote the number of arms, $\mu$ to denote the mean rewards, and $\Delta$ to denote the (mean) reward gap between the best and the second-best best arms. As we will work on arms with Bernoulli distributions, we use $\bern{\mu}$ to denote the Bernoulli distribution with mean $\mu$, i.e., with probability $\mu$ the realization is $1$. 



\subsection{The Multi-pass Streaming MABs Model} 
\label{subsec:model}
We use the streaming MABs model introduced by \cite{AssadiW20} and extended by \cite{JinH0X21,AgarwalKP22}. Informally, the model assumes $n$ arms arriving in a stream with an adversarial order. For each arriving arm, the algorithm is allowed to pull the arriving arm and the stored arms for an arbitrary number of times. After the arm pulls, the algorithm can $(i)$. store the arriving arm; $(ii)$ discard the arriving arm; and $(iii)$. discard stored arms from memory. If an arm is discarded, it will not be available until its appearance in the next pass of the stream. We further assume the order of arrival is \emph{fixed} across different passes. We define the \emph{sample complexity} as the number of total arm pulls used by an algorithm, and the \emph{memory complexity} as the maximum number of arms ever stored at any point in the memory. For the purpose of the lower bound proof, we allow the algorithm to store any \emph{statistics} for free \footnote{Any algorithm with unbounded memory can simulate the ones with bounded statistics, and we have no rescrition on local computation power. As such, our lower bound also applies to algorithms with limited memory for statistics. }.

We give a formalization of the above description in what follows. We first define the deterministic algorithms before extending the notion to randomized algorithms. Let $\{\arm_{i}\}_{i=1}^{n}$ be $n$ arms with Bernoulli distributions of means $\{\mu_{i}\}_{i=1}^{n}$, i.e., the distribution for $\arm_{i}$ is $\bern{\mu_{i}}$\footnote{We work with Bernoulli distributions for a \emph{lower bound} proof that applies to all sub-Gaussian reward distributions.}. The arms arrive one-by-one in a stream, whose order is specified by a permutation $\sigma$ on $[n]$. We say $\ALG$ is a $P$-pass (deterministic) streaming algorithm with an $s$-arm memory if 
\begin{itemize}
\item $\ALG$ maintains two objects:
    \begin{enumerate}
    \item Memory $M \subseteq \{1,2,\cdots, n, \perp\}^{s}$ and a buffer index $\ja \in [n]$ for the arriving arm. We denote $\sM$ as the random variable\footnote{Although the algorithm is deterministic, there is inherent randomness from arm pulls.} for $M$ and $\bM$ for the set of all possible memory $M$.
    \item Transcript $\Pi = ([P], [n], [n], \{0,1\})^{*}$, which is an ordered list of \emph{tuples}, and each tuple encodes the index of the pass, the index of the arriving arm ($\ja$), the index of the pulled arm, and the result of a single arm pull. We further denote $\sPi$ as the random variable for $\Pi$ and $\mathbf{\Pi}$ as the set of all possible transcripts.
    \end{enumerate}
\item $\ALG$ has an access to a sampler $\cO: \{\arm_{\sigma(i)}\mid i\in M\} \cup \{\arm_{\sigma(\ja)}\} \rightarrow \{0,1\}$ that can be repeatedly use to make a single arm pull among the stored arms and the arriving arm. After a call of $\cO$ on the $\sigma(i)$ arm, we add tuple $(p, \ja, \sigma(i), x)$ to the transcript $\Pi$, where $x\in\{0,1\}$ is the outcome of the arm pull.
\item $\ALG$ has an update function $\cU: \bM \times [n] \times [n] \times [P] \times \mathbf{\Pi} \rightarrow \bM$ that takes memory $M$, the index of the arriving arm $\ja$, the index of the sampled arm $i$, the current pass index $p\in [P]$, the past transcript $\Pi$, and the sampler $\cO$, outputs a new memory state $M$. 
\end{itemize}
With the above formalization, we can define the sample complexity $\sample{\ALG}$ (total number of arm pulls) as the total number of times $\cO$ is called, and $\memory{\ALG}=s$ as the maximum number of indices that can be stored (minus the one-arm buffer) at any point.

\paragraph{Randomized algorithms.} We can extend the above notion of $P$-pass deterministic streaming algorithms to \emph{randomized} algorithms in the standard manner. Concretely, a randomized algorithm with the set of internal random bits $\mathcal{R}$ can be viewed as a \emph{distribution} over deterministic algorithms: for each $r\in \mathcal{R}$, there is a realization of a deterministic $P$-pass streaming algorithm. Note that similar to the storing of statistics, we do \emph{not} charge the space for random bits, i.e., the algorithms can store an unlimited number of internal random bits for free. Since we are able to prove a lower bound under this setting, we can natrually extend the lower bound to algorithms with limited random bits.

\paragraph{Offline algorithms.} To unify the arguments in the rest of the paper, we can define offline (i.e., classical RAM) algorithms as \emph{simulations} of streaming algorithms under the above framework. Concretely, we can view the offline algorithm as a single-pass streaming algorithm that uses a memory of $n$ arms. It first reads and stores all arms, and then makes calls on the sampler $\cO$. Note also that an offline algorithm is able to simulate the \emph{passes} and the \emph{indices} of arms locally, i.e., to use the local memory to make an arbitrary number of (extra) passes over the stream and read an arbitrary number of arms before calling the sampler $\cO$ with a desired $\ja$. As such, the transcript of an offline algorithm can be written as ordered tuples of $\Pi = (*, *, [n], \{0,1\})^{*}$, where the first two elements can be modified to any index in $[P]$ and $[n]$.

\subsection{Standard Sample Complexity Lower Bounds for Single-armed Bandit}
\label{subsec:single-arm-complexity}
We present lower bounds on the necessary number of arm pulls to identify the reward of a \emph{single} arm. These lower bounds serve as the basis for the reduction proofs we used in the auxiliary lemmas (\Cref{lem:arm-trapping-new} and \Cref{lem:sample-success-tradeoff}). We remark that the lemmas are not limited to the streaming setting and they hold even for classical algorithms. Since these lemmas are variates of known results, we defer their proofs to \Cref{app:single-arm-proof}. 

Our first lemma shows that to \emph{determine} the mean reward of an arm from distributions with gap $\beta$, an $\Omega(1/\beta^2)$ number of arm pulls is necessary.

\begin{lemma}
\label{lem:arm-identify}
Consider an arm with a Bernoulli distribution whose mean is parameterized as follows.
\begin{itemize}
\item With probability $\rho$, the mean reward is $\frac{1}{2}+\beta$;
\item With probability $1-\rho$, the mean reward is $\frac{1}{2}$;
\end{itemize}
where $\rho\in (0,\frac{1}{2}]$ is a fixed parameter. Any algorithm to determine the reward of the arm for $\beta\in (0, \frac{1}{6})$ and a success probability of at least $(1-\rho+\eps)$ has to use $\frac{1}{4}\cdot \frac{\eps^2}{\rho^2 \beta^{2}}$ arm pulls.
\end{lemma}

We further provide a lemma showing that if the number of arm pulls is small, the ``knowledge'' of the algorithm cannot change the distribution for which case the instance is from by too much. More formally, we prove that with a limited number of arm pulls, from the algorithm's perspective, the probability for which case the instance is from remains close to the original distribution.
\begin{lemma}
\label{lem:arm-learn}
Let $\beta \in (0,\frac16)$ and $\rho \in (0,\frac12)$. Sample $\Theta$ from $\set{0,1}$ such that $\Theta=1$ with probability $\rho$.
Consider an arm with a Bernoulli distribution from the following family:
\begin{itemize}
\item If $\Theta=1$, the distribution is $\bern{1/2+\beta}$;
\item If $\Theta=0$, the distribution is $\bern{1/2}$;
\end{itemize}
Let $\ALG$ be an algorithm that uses at most $s=\frac{1}{12}\cdot \frac{\eps^3}{\rho \cdot \beta^{2}}$ arm pulls on an instance $I$ sampled from the family. Let $\Pi$ be the transcript of $\ALG$, and let $\sPi$ be the random variable of $\Pi$. Then, with probability at least $1-\eps$ over the randomness of transcript $\sPi$, there is
\begin{align*}
& \Pr\paren{\Theta=1 \mid \sPi=\Pi} \in [\rho -  \eps,  \rho +  \eps]; \\
& \Pr\paren{\Theta=0 \mid \sPi=\Pi} \in [1-\rho - \eps,  1- \rho + \eps].
\end{align*}
\end{lemma}




\newcommand{\eventF}{\ensuremath{\event_{\textnormal{First}}}}
\newcommand{\DeltaT}{\ensuremath{\widetilde{\Delta}}\xspace}

\section{Main Result}
\label{sec:adversary}
We show the formal statement of our main result in this section. We note that the formalization of \Cref{rst:main-result} requires some work, and in particular, we need to specify the meaning of the `lack of knowledge' on $\Delta$ by the algorithm. To this end, we define the distribution $\cD(P,C)$ of MAB instances  
for any two arbitrary integers $P\geq 2$, $C \geq 1$ as follows (roughly speaking, $P$ corresponds to the number of passes of the streaming algorithms, and $C$ is the hidden-constant in the sample complexity of the algorithm -- this will become clear shortly)\footnote{We focus on $P\geq 2$ for technical reasons. For $P=1$, \cite{AWneurips22} already proved that the sample complexity is unbounded when using $o(n)$-arm memory.}. An illustration of the construction of $\cD(P,C)$ can be found in \Cref{fig:instance}. 



\begin{tbox}
\textbf{Distribution $\cD(P,C)$}: A family of ``hard'' MAB instances for $P$-pass streaming algorithms.
\begin{enumerate}
\item Divide the $n$ arms into $(P+1)$ batches $B_1,\ldots,B_{P+1}$ with equal sizes of $b:= \frac{n}{P+1}$. The batches are ordered in \emph{reverse} of the stream, i.e., in each pass, $B_{P+1}$ arrives first, then $B_P$, all the way to $B_1$ that arrives last. 
\item Initialize all the arms in every batch to have mean reward $1/2$. 

\item For any batch $B_p$ for $1 \leq p \leq P$: 
\begin{enumerate}
	\item Sample a coin $\Theta_p \in \set{0,1}$ from the Bernoulli distribution $\bern{\frac{1}{2P}}$.
	\item If $\Theta_p = 1$, then sample an arm uniformly at random from the batch $B_p$ and change its mean reward to $1/2+\eta_p$ for a parameter $\eta_p$ defined as: 
	\begin{align}
		\eta_p := \paren{\frac{1}{6 \, C \cdot P}}^{15p}. \label{eq:gap-parameter} 
	\end{align}
	We refer to this arm as the \textbf{special arm} of batch $B_p$ (which only exists if $\Theta_p = 1$). 
\end{enumerate}
\item For the batch $B_{P+1}$: 
\begin{enumerate}
	\item Sample an arm uniformly at random from $B_{P+1}$ and change its mean reward to $1/2+\eta_{P+1}$ for $\eta_{P+1}$ as defined in \Cref{eq:gap-parameter}. 
	We refer to this arm as the \textbf{special arm} of batch $B_{P+1}$ (which always exists) and denote it by $\armstar_{P+1}$.  
\end{enumerate}
\end{enumerate}
\end{tbox}

\begin{figure}[h!]
	\centering
	\includegraphics[width=0.9\textwidth]{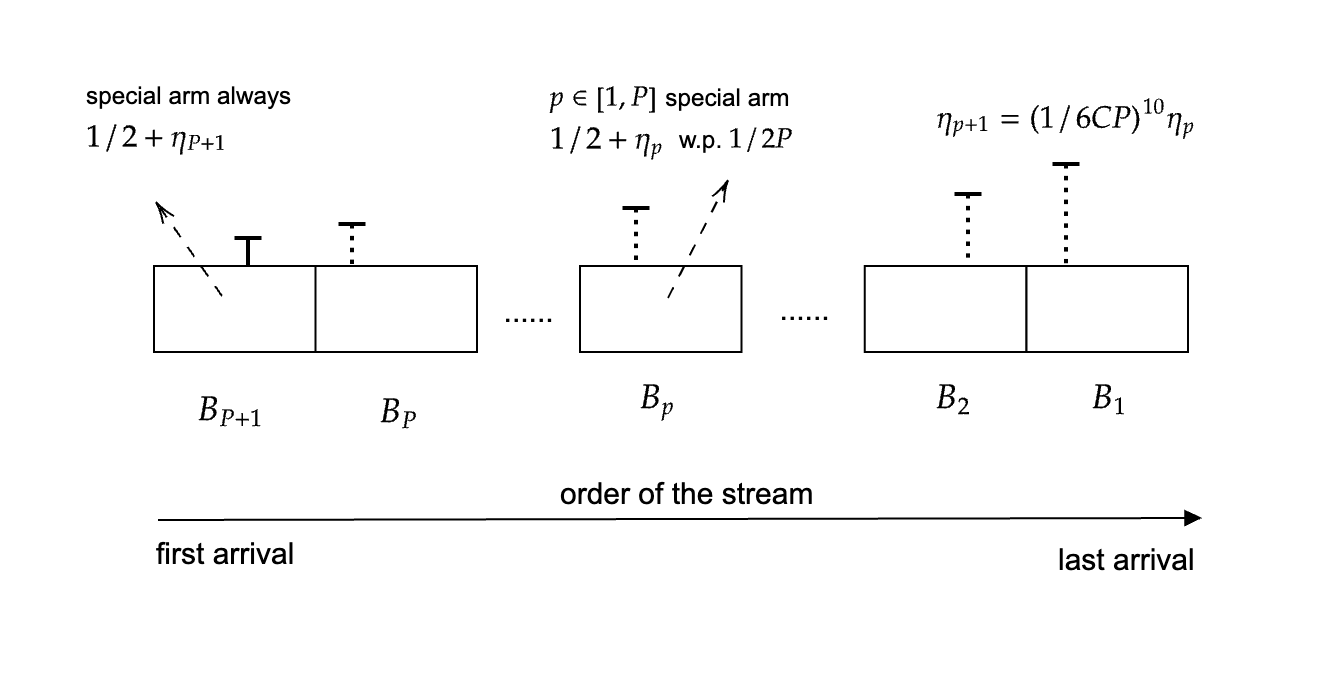}
	\caption{\label{fig:instance} An illustration of $\cD(P,C)$. The indices of batches are arranged in the \emph{reversed} order of the arrival of the stream. Batch $B_{P+1}$ always has an arm with $1/2+\eta_{P+1}$ mean reward, while other batches $p$ has its special arm with mean reward $1/2+\eta_{p}$ with probability $\frac{1}{2P}$.}
\end{figure}

To continue, we need some notation. We use $I \sim \cD(P,C)$ to denote an instance of streaming MAB sampled from the distribution $\cD(P,C)$. 
For any instance $I$, we define $\Delta(I)$ to denote the gap between the best and second best arm. Moreover, for any instance $I$ and integer $p \in [P+1]$, we define the following event: 
\begin{itemize}
	\item $\eventF(p)$: the variables $\Theta_1 = \Theta_2 = \cdots = \Theta_{p-1} = 0$ (shorthand, $\Theta_{<p} = 0$), but $\Theta_p = 1$ (with a slight abuse of notation, we take $\Theta_{P+1}$ to be a deterministic variable which is always $1$). 
\end{itemize}
Notice that the events $\eventF(1),\ldots,\eventF(P+1)$ are mutually exclusive and exactly one of them happens for any instance. We define the \textbf{special batch} of an instance $I$ as the batch $B_p$ for the value of $p \in [P+1]$ where $\eventF(p)$ happens. 

The following observation shows that the parameter $\Delta$ of an instance $I$ is basically determined by the choice of the special batch. 
\begin{observation}\label{obs:Delta}
	For any $I \sim \cD(P,C)$, if the special batch of $I$ is $B_p$ for $p \in [P+1]$, then 
	\[
		\frac{1}{2} \cdot \eta_p \leq \Delta(I) \leq \eta_p. 
	\]
\end{observation}
\begin{proof}
Note that by our construction, the best arm is the special arm of the special batch. Let $p$ be the index of the special batch, i.e. $B_{p}$ is the first batch such that $\Theta_{p}=1$. We prove the upper and lower bounds separately:
\begin{enumerate}
\item Upper bound: Observe that there exist (many) arms with mean reward $\frac{1}{2}$, which create a gap of $\eta_{p}$ w.r.t. the best arm. Since $\Delta(I)$ is the \emph{smallest} gap w.r.t. the best arm, we have $\Delta(I)\leq \eta_{p}$.
\item Lower bound: Observe that when $\Theta_{p}=1$, the (potentially existing) arm with the closest mean reward is with reward $\frac{1}{2}+\eta_{p+1}$. As such, the value of $\Delta(I)$ is at least 
\[ \eta_{p}-\eta_{p+1}=\paren{1-(\frac{1}{6CP})^{15}} \cdot \eta_{p}\geq \frac{1}{2}\cdot \eta_{p},\]
where the last inequality is obtained by using $C\geq 1$ and $P\geq 2$.
\end{enumerate}
Combining the above gives us the desired bounds. 
\end{proof}

By \Cref{obs:Delta}, the value of $\Delta$ varies based on the realization of $\eventF(p)$ with different $p$ values. As such, if an algorithm can achieve the optimal sample complexity bound without the knowledge of $\Delta$ given a priori, it must `adjust' its sample complexity to be competitive with $O(n/\eta^2_{p})$ \emph{if $\eventF(p)$ happens}. This requirement and its consequence can be formalized in our main technical theorem as follows.

\begin{theorem}\label{thm:main}
	For any integers $P\geq 2$, $C \geq 1$, the following is true. Let $\ALG$ be any deterministic $P$-pass streaming algorithm that uses a memory of $\memory{\ALG} \leq {n}/{(20000 P^3)}$ arms. 
	
	Suppose the following is true for $\ALG$  on instances of distribution $\cD(P,C)$ and every $p \in [P+1]$: 
	\[
		\Exp\bracket{\sample{\ALG} \mid \eventF(p)} \leq C \cdot \frac{n}{\eta_p^2}, 
	\]
	where the randomness is taken over the choice of the instance $I \sim \cD(P,C) \mid \eventF(p)$ and the arm pulls. 
	Then, the probability that $\ALG$ can output the best arm for $I \sim \cD(P,C)$ is strictly less than $999/1000$. 
\end{theorem}

\Cref{thm:main} implies that for a streaming algorithm to find the best arm with a good probability and without the a priori knowledge of $\Delta$, it cannot simultaneous achieve $i).$ a low number of passes, $ii).$ a low memory, and $iii).$ the ability to `adjust' the sample complexity to compete with the optimal bound. As such, combining \Cref{thm:main} with \Cref{obs:Delta} formalizes our \Cref{rst:main-result} in the introduction.

\begin{corollary}[Formalization of \Cref{rst:main-result}]
	\label{cor:main}
	For any $\DeltaT>0$, there exists a family of streaming MABs instances $\dist$ in which every instance has $\Delta\geq \DeltaT$, such that any streaming algorithm (deterministic or randomized) that finds the best arm with an \emph{expected} sample complexity of $O(n/\Delta^2)$, a success probability of at least $1999/2000$, and a
	space of $o(n/\log^3{(1/\DeltaT)})$ arms requires $\Omega(\frac{\log{(1/\DeltaT)}}{\log\log{(1/\DeltaT)}})$ passes over the stream. 
\end{corollary}
\begin{proof}
We first prove the statement for deterministic algorithms on the distribution $\dist(P,C)$ with success probability $999/1000$. 
Our proof strategy is as follows. For any $\DeltaT$ and any algorithm that uses $C' \cdot \frac{n}{\Delta^2(I)}$ arm pulls for arbitrary constant $C'$, we pick appropriate $P$ based on $\DeltaT$ and $C$ based on $C'$. Then, we sample an instance from $\dist(P,C)$, and argue that the properties in \Cref{cor:main} matches the properties prescribed in \Cref{thm:main} -- in particular, if the algorithm always uses $C' \cdot \frac{n}{\Delta^2(I)}$ arm pulls in expectation, the expected arm pulls of $\Exp[\sample{\ALG} \mid \eventF(p)]$ is at most $4C' \cdot \frac{n}{\eta^2_{p}}$. Finally, it turns out that the value of $P$ is at least $\Omega\paren{\frac{\log{(1/\DeltaT})}{\log\log{(1/\DeltaT)})}}$ by this construction, which gives the desired lower bound.

We now formalize the above strategy. For a streaming algorithm that uses $C' \cdot \frac{n}{\Delta^2(I)}$ samples, we pick $C=4\cdot C'$ and use the distribution $\dist(P,C)$ as the adversarial family of instances. We further choose $P = \Omega(\frac{\log{(1/\DeltaT)}}{\log\log{(1/\DeltaT)}})$, and observe the following properties:
\begin{itemize}
\item If the event $\eventF(p)$ happens in $\dist(P,C)$, the algorithm takes at most $C\cdot \frac{n}{\eta^2_{p}}$ arm pulls. To see this, note that by the upper bound of the expected number of samples, and conditioning on $\eventF(p)$, there is
\[
	\Exp\bracket{\sample{\ALG} \mid \eventF(p)} \leq C'\cdot \frac{n}{\Delta(I)^2} = \frac{C}{4} \cdot \frac{n}{\Delta(I)^2} \leq \frac{C}{4} \cdot \frac{n}{(\eta_p/2)^2} = C \cdot \frac{n}{\eta_p^2}, 
\]
where the first inequality follows from ~\Cref{obs:Delta}.
\item For any $p\in[P+1]$, there is $\Delta \geq \DeltaT$ and $\eta_{p}> 2\cdot \DeltaT$, i.e., 
\[
	\Delta \geq \frac{\eta_p}{2} \geq \frac{\eta_{P+1}}{2} = \frac{1}{2}\cdot \paren{\frac{1}{6 \, C \cdot P}}^{15P+15} > \DeltaT, 
\]
where the last inequality is obtained by plugging in $P = \Omega\paren{\frac{\log{(1/\DeltaT})}{\log\log{(1/\DeltaT)})}}$.
\end{itemize}

Note that the above conditions imply $(i).$ the expected number of samples follows the upper bound of \Cref{thm:main}; $(ii).$ the memory follows the upper bound of \Cref{thm:main} since $o(n/(\log{(1/\DeltaT)})^3)=o(n/P^3)$ by the choice of $P$; and $(iii).$ the condition of $\Delta \geq \DeltaT$ is satisfied. 
As such, we can apply \Cref{thm:main}, and show that the algorithm must make at least $P=\Omega(\frac{\log{(1/\tilde{\Delta})}}{\log\log{(1/\tilde{\Delta})}})=\Omega(\frac{\log{(1/\Delta)}}{\log\log{(1/\Delta)}})$ passes over the stream for any instance in the family, which proves the corollary for deterministic algorithms.

We now extend the result to randomized algorithms. This is a standard application of Yao's minimax principle, and we provide the proof for completeness. Assume for the purpose of contradiction that there exists a randomized algorithm with a success probability of $1999/2000$ and the same restrictions as the deterministic algorithms. 
Let $\mathbf{R}$ be the set of internal randomness, and define $r\in \mathbf{R}$ as a \emph{good} random string if
$\Pr(\text{\ALG returns the wrong arm} \mid r) \leq 1/1000$,
where the randomness is over the inputs and the arm pulls. We say $r$ is a \emph{bad} random string if the above inequality does not hold. By the success probability of the algorithm, there is
\begin{align*}
	\exprand{r \in \mathbf{R}}{\Pr(\text{\ALG returns the wrong arm} \mid r)} \leq \frac{1}{2000}.
\end{align*}
Therefore, by the Markov bound, we have 
\begin{align*}
	\Pr(r\in\mathbf{R} \text{ is good}) \geq \frac{1}{2}.
\end{align*}
As such, the expected sample complexity can be written as
\begin{align*}
	\expect{\sample{\ALG}} &\geq \expect{\sample{\ALG}\mid r\in\mathbf{R} \text{ is good}}\cdot \Pr(r\in\mathbf{R} \text{ is good})\\
	&= \frac{1}{2}\cdot \expect{\sample{\ALG}\mid r\in\mathbf{R} \text{ is good}}.
\end{align*}
As such, by the expected sample complexity $\expect{\sample{\ALG}}\leq O(n/\Delta^2)$ of the randomized algorithm, we have $\expect{\sample{\ALG}\mid r\in\mathbf{R} \text{ is good}} \leq O(n/\Delta^2)$. Therefore, for any \emph{good} choice of $r$, we obtain a \emph{deterministic} algorithm that uses $O(n/\Delta^2)$ arm pulls, a success probability of at least $999/1000$, and the memory restriction of $o(n/(\log{(1/\DeltaT)})^3)$ -- which reaches a contradiction with the lower bound for the deterministic algorithm.
\end{proof}

Note that in \Cref{cor:main}, the memory is fixed, but the sample complexity and the number of passes are allowed to be random. 
As long as $\DeltaT \geq 2^{n^{1/3-\Omega(1)}}$, the result matches the upper bound of \cite{JinH0X21} up to an exponentially smaller term.

 



 The rest of this paper is dedicated to the proof of \Cref{thm:main}. In the next section, we state some auxiliary information-theoretic lemmas in the context of finding best arm, or rather ``trapping'' it, outside the streaming model. 
Afterwards, we present the main part of our argument that uses these lemmas to establish a streaming lower bound and prove \Cref{thm:main}.

\begin{remark}
We use a success probability of $\frac{1999}{2000}$ in \Cref{rst:main-result} for technical convenience. For lower bounds with a lower success probability, we can apply the standard reduction argument that ``boosts'' the success probability. Concretely, suppose we have a $P$-pass algorithm with $s$ samples, $m$-arm memory, and a success probability $\frac{1}{2}+\eps$ for any $\eps=\Omega(1)$. In our distribution $\cD(P,C)$, we can obtain the value of $\Delta$ by the end of pass $P$. Therefore, we can run $O(1)$ streams in parallel, and spend $O(\frac{1}{\Delta^{2}})$ samples in the end to return the arm with the best empirical rewards. Such an algorithm has a success probability of ${1999}/{2000}$, an $O(m)$-arm memory, and uses $O(s)$ samples. Hence, the asymptotical sample-memory-pass trade-off remains valid for algorithms with $\frac{1}{2}+\eps$ success probability for any $\eps=\Omega(1)$. 
\end{remark}

\newcommand{\istar}{\ensuremath{i^*}\xspace}


\section{Auxiliary Lemmas for Pure Exploration in MABs}
\label{sec:tech-lemma}
We present two auxiliary lemmas in this section that are needed for our main proof. These lemmas concern MABs for \emph{offline} algorithms, i.e., without any streaming restriction, and they can be used in the streaming setting with arbitrary pass and $\ja$ index (see \Cref{subsec:model} for the detailed discussion).
The proofs are rather standard
application of known ideas. However, we are not aware of an exact formulation of these lemmas in prior work that we need in our main proofs in the subsequent section; thus, for completeness, we present and prove these lemmas in this section. 

The first lemma is a generalization of the arm-trapping lemma of~\cite{AWneurips22} to the case when success probability can be quite small. 



\begin{lemma}[low-probability arm-trapping lemma]
\label{lem:arm-trapping-new}
Suppose we have a set of $k \geq 1$ arms with mean reward $1/2$ and we pick one of them uniformly at random -- called the \textbf{special arm} -- and increase its reward to $1/2+\beta$ for some $\beta > 0$. 

For any parameter $\gamma \in (0, \frac{1}{2}]$, any algorithm that outputs a set $S$ of $(\gamma\cdot k/12)$ arms such that with probability at least $\gamma$ the special arm belongs to $S$ 
requires $\frac{1}{300} \cdot \frac{\gamma^3}{\beta^2} \cdot k$ arm pulls. 
\end{lemma}
\begin{proof}
We give a reduction proof in the `direct-sum' style in the same spirit of \cite{AWneurips22}. In particular, we show that if there exists an algorithm with less than $\frac{1}{300} \cdot \frac{\gamma^3}{\beta^2} \cdot k$ sample that `traps' the special arm, we can turn it into an algorithm that contradicts the sample lower bound in \Cref{lem:arm-identify} with $\rho=\frac{1}{2}$ and $\eps=\frac{\gamma}{6}$ (which would require $\frac{1}{36}\cdot \frac{\gamma^2}{\beta^2}$ arm pulls by \Cref{lem:arm-identify}). The reduction goes as follows.

\begin{tbox}
\textbf{Inputs:} 
\begin{enumerate}[label=\alph*)]
\item A single $\armtilde$ with the mean reward following the distribution in \Cref{lem:arm-identify} with $\rho=\frac{1}{2}$; 
\item An algorithm $\ALG$ that outputs a set $S$ of $(\gamma\cdot k/2)$ arms such that $(i).$ with probability at least $\gamma$, the special arm belongs to $S$; $(ii).$ $\ALG$ uses less than $\frac{1}{300} \cdot \frac{\gamma^3}{\beta^2} \cdot k$ arm pulls.
\end{enumerate}

\medskip

\textbf{Procedure:} 

\begin{enumerate}
\item\label{line:blind-no-case} With probability $(\frac{1}{2}-\frac{\gamma}{3})$, output \emph{``reward of $\armtilde$ is $\frac{1}{2}+\beta$''}.
\item\label{line:follow-alg-case} With probability $(\frac{1}{2}+\frac{\gamma}{3})$, output with the following procedure:
\begin{enumerate}[label=(\roman*)]
\item Create $k-1$ `dummy arms' and let their reward mean be $\frac{1}{2}$. 
\item Put $\armtilde$ uniformly at random on index $\istar$ among the $k$ arms, and run $\ALG$.
\item\label{line:sample-break} If $\ALG$ uses more than $\frac{1}{37}\cdot \frac{\gamma^2}{\beta^2}$ arm pulls on $\armtilde$, abort $\ALG$ and output \emph{``reward of $\armtilde$ is $\frac{1}{2}+\beta$''}.
\item\label{line:trap-check} Otherwise, if the output of $S$ contains index $\istar$, output \emph{``reward of $\armtilde$ is $\frac{1}{2}+\beta$''}; if the output of $S$ does \emph{not} contain index $\istar$, output \emph{``reward of $\armtilde$ is $\frac{1}{2}$''}.
\end{enumerate}
\end{enumerate}
\end{tbox}

It is straightforward to see that the algorithm never uses more than $\frac{1}{37}\cdot \frac{\gamma^2}{\beta^2}$ arm pulls on the special arm, as we directly terminate the process whenever it uses more arm pulls. It remains to verify the correctness of distinguishing the cases.

\paragraph{Case A): the true reward of $\armtilde$ is $1/2$.} Due to Line~\ref{line:blind-no-case}, there is a probability of $\frac{1}{2}-\frac{\gamma}{3}$ that the reduction never outputs ``reward of $\armtilde$ is $\frac{1}{2}$''. Nevertheless, we will eventually show that when the algorithm does \emph{not} enter Line~\ref{line:blind-no-case}, the marginal correct probability is high enough to guarantee an overall $\frac{1}{2}+O(\gamma)$ correct probability.  

Let $s_{i}$ be the number of samples that uses on an arm with index $i$. Note that in this way, $s_{\istar}$ stands for the number of samples used for $\armtilde$. Since the index of $\istar$ is chosen uniformly at random, there is
\begin{align*}
	\expect{s_{\istar}} & = \sum_{i=1}^{k} \Pr\paren{\istar=i}\cdot \expect{s_{i^{*}}\mid \istar=i} \\
	& = \frac{1}{k}\cdot \sum_{i=1}^{k} \expect{s_{i}} \\
	& =  \frac{1}{k}\cdot \expect{\sum_{i=1}^{k}  s_{i}} \tag{by linearity of expectation}\\
	& \leq \frac{1}{300}\cdot \frac{\gamma^3}{\beta^2} \tag{by the sample upper bound of $\ALG$}.
\end{align*}
Therefore, by a Markov bound, we can upper-bound the probability for the special arm to use more than $\frac{1}{37}\cdot \frac{\gamma^2}{\beta^2}$ arm pulls by
\begin{align*}
\Pr\paren{s_{\istar}\geq \frac{1}{37}\cdot \frac{\gamma^2}{\beta^2}}\leq \frac{\gamma}{8}.
\end{align*}
As such, the probability for the reduction to false report reward as $\frac{1}{2}+\beta$ from Line~\ref{line:sample-break} is at most $\frac{\gamma}{5}$. Furthermore, in line~\ref{line:trap-check}, since the arms are identical random variables and the index $\istar$ is chosen uniformly at random, there is
\begin{align*}
\Pr\paren{\text{$S$ contains index } {\istar}} & = \Pr\paren{\text{$S$ contains index } {i},\, \forall i} = \frac{\card{S}}{k} = \frac{\gamma}{8}.
\end{align*}
Therefore, the probability for the reduction to falsely output ``reward of $\armtilde$ is $\frac{1}{2}+\beta$'' through the output of $\ALG$ on line~\ref{line:trap-check} is at most $\frac{\gamma}{8}$. As such, we have
\begin{align*}
\Pr\paren{\text{$\ALG$ outputs ``reward is $\frac{1}{2}$''} \mid \text{reward is $\frac{1}{2}$, \, Line~\ref{line:follow-alg-case} happens}} \geq 1-\frac{\gamma}{8}-\frac{\gamma}{8} = 1-\frac{\gamma}{4}. \tag{by union bound}
\end{align*}
As such, the probability for the reduction to succeed when the special arm is with reward $\frac{1}{2}$ is at least:
\begin{align*}
& \Pr\paren{\text{$\ALG$ outputs ``reward is $\frac{1}{2}$''}\mid \text{reward is $\frac{1}{2}$}} \\
& \quad = \Pr\paren{\text{$\ALG$ outputs ``reward is $\frac{1}{2}$''} \mid \text{reward is $\frac{1}{2}$, \, Line~\ref{line:follow-alg-case} happens}}\cdot \Pr\paren{\text{Line~\ref{line:follow-alg-case} happens}}\\
& \quad \geq (\frac{1}{2}+\frac{\gamma}{3})\cdot (1- \frac{\gamma}{8}- \frac{\gamma}{8})\\
&\quad \geq \frac{1}{2}+\frac{\gamma}{6}. \tag{using $\gamma\leq 1/2$}
\end{align*}

\paragraph{Case B): the true reward of $\armtilde$ is $1/2+\beta$.} In this case, there is a probability of $(\frac{1}{2}-\frac{\gamma}{3})$ that the algorithm simply outputs ``reward of $\armtilde$ is $\frac{1}{2}+\beta$'' from Line~\ref{line:blind-no-case}. Furthermore, in the case of Line~\ref{line:follow-alg-case}, the reduction succeed with a probability that is at least as large as $\gamma$ by the guarantee of the trapping algorithm $\ALG$. As such, the success probability in this case is at least
\begin{align*}
\Pr\paren{\text{$\ALG$ outputs ``reward is $\frac{1}{2}+\beta$''}\mid \text{reward is $\frac{1}{2}+\beta$}} &\geq \frac{1}{2}-\frac{\gamma}{3} + (\frac{1}{2}-\frac{\gamma}{3})\cdot \gamma\\
&\geq \frac{1}{2} + \frac{\gamma}{6}.
\end{align*}

Summarizing the cases of $A)$ and $B)$ establishes the correctness of the reduction for $\rho=\frac{1}{2}$ and $\eps=\frac{\gamma}{6}$. By \Cref{lem:arm-identify}, the sample complexity has to be at least $\frac{1}{36}\cdot \frac{\gamma^2}{\beta^2}$, which contradicts the $\frac{1}{37}\cdot \frac{\gamma^2}{\beta^2}$ sample complexity and proves the lemma.
\end{proof}

The second lemma uses a distribution similar to \Cref{lem:arm-trapping-new}, albiet the special arm is now allowed to be ``flat'' -- with mean reward $\frac{1}{2}$ -- with probability $1-\alpha$. The lemma says that if the number of used arm pulls is small, the internal distribution (the ``knowledge'') of the algorithm on whether the special arm is ``flat'' remains close to the original, i.e., with probability $\sim (1-\alpha)$.

\begin{lemma}[A sample-knowledge trade-off lemma]
\label{lem:sample-success-tradeoff}
Consider the following distribution $\dist$ on $k \geq 1$ arms for some parameters $\alpha,\beta > 0$ and $\alpha<\frac{1}{2}$: 
\begin{itemize}
	\item \textbf{No} case: with probability $\alpha$, all except for one uniformly at random chosen arm have mean reward $1/2$, while the chosen arm have reward $1/2+\beta$;
	\item \textbf{Yes} case: with probability $1-\alpha$, all the arms have mean reward $1/2$.
\end{itemize}
Suppose we have an algorithm that given an instance $I$ sampled from this distribution $\distmu$ makes at most $\frac{1}{100}\cdot \frac{\gamma^{2} \cdot k}{\alpha \cdot \beta^2}$ arm pulls for some $\gamma \in (0,\frac{1}{5}]$. 
Let $\sPi$ and $\Pi$ be the random variable and the realization of the transcript. 
Then, with probability at least $(1-2 \, \gamma^{1/2})$ over the randomness of $\sPi$,
\begin{align*}
	&\Pr_{I}\paren{\text{$I$ is a \textbf{No} case} \mid \sPi = \Pi} \in [\alpha - 2\cdot\gamma^{1/2}, \alpha + 2\cdot \gamma^{1/2}]; \\
	&\Pr_{I}\paren{\text{$I$ is a \textbf{Yes} case} \mid \sPi = \Pi} \in [1-\alpha - 2\cdot \gamma^{1/2}, 1-\alpha + 2\cdot \gamma^{1/2}]. 
\end{align*}
\end{lemma}

\begin{proof}
We again prove the lemma by a `direct-sum' type of reduction. In particular, we show that for a family of arms distributed as prescribed by \Cref{lem:sample-success-tradeoff}, any algorithm that learns the distribution of $I$ with $\eps$ advantage over random guessing and $s$ arm pulls can learn the distribution in \Cref{lem:arm-learn} with $O(\eps)$ advantage and $O(\frac{s}{k}\cdot \poly(\frac{1}{\eps}))$ arm pulls. This allows us to eventually build a contradiction towards \Cref{lem:arm-learn} with $\rho=\alpha$ and $\eps=2\gamma^{1/2}$.

\begin{tbox}
\textbf{Inputs:} 
\begin{enumerate}[label=\alph*)]
\item A single arm with mean reward following the distribution in \Cref{lem:arm-learn} with $\rho=\alpha$; 
\item An algorithm $\ALG$ that outputs $\Pr_{I}\paren{\text{$I$ is a \textbf{No} case} \mid \sPi = \Pi}$ as in Lemma~\ref{lem:sample-success-tradeoff}.
\end{enumerate}

\medskip

\textbf{Procedure:} 
\begin{enumerate}
\item Create $k-1$ `dummy arms' and let their mean reward be $\frac{1}{2}$. 
\item Put the special arm uniformly at random at index $\istar$ among the $k$ arms, and run $\ALG$.
\item\label{line:learn-break} If the special arm uses more than $\frac{1}{5}\cdot \frac{\gamma^{3/2}}{\beta^2\alpha}$ arm pulls, stop the algorithm and output \textbf{No}.
\item\label{line:learn-check} Otherwise, set the probability of $\Pr\paren{\Theta=1 \mid \sPi=\Pi}$ in \Cref{lem:arm-learn} (i.e., the arm is from $\bern{1/2+\beta}$) as the same with $\Pr_{I}\paren{\text{$I$ is a \textbf{No} case} \mid \sPi = \Pi}$.
\end{enumerate}
\end{tbox}
We focus on the case of the upper bound of $\Pr_{I}\paren{\text{$I$ is a \textbf{No} case} \mid \sPi = \Pi}$ since the lower bound follows from the same logic.
Suppose for the purpose of contradiction that $\ALG$ uses at most $\frac{1}{100}\cdot \frac{\gamma^{2} \cdot k}{\alpha \cdot \beta^2}$ arm pulls and achieves 
\[\Pr_{I}\paren{\text{$I$ is a \textbf{No} case} \mid \sPi = \Pi}>\alpha+2\cdot \gamma^{1/2}.\]
By letting $\eps = 2\gamma^{1/2}$, the reduction deterministically uses at most $\frac{1}{5}\cdot \frac{\gamma^{3/2}}{\beta^2\alpha}=\frac{1}{40}\cdot \frac{\eps^3}{\beta^2\alpha} <\frac{1}{12}\cdot \frac{\eps^3}{\beta^2\alpha}$ arm pulls as we terminate whenever it uses more. We now show that with the reduction, there is
\[\Pr_{\Pi}\paren{\Pr(\Theta=1 \mid \sPi = \Pi) > \alpha+2\cdot \gamma^{1/2}} \geq 1-2\gamma^{1/2},\]
which leads to a contradiction with \Cref{lem:arm-learn} with our choice of $\eps$. 

Note that $I$ is a \textbf{No} case if and only if $\Theta=1$ (the special arm is with mean reward $\frac{1}{2}+\beta$). As such, if the reduction reaches Line~\ref{line:learn-check}, then by the guarantee of the algorithm $\ALG$, there is 
\begin{align*}
\Pr_{I}(\text{$I$ is a \textbf{No} case} \mid \sPi = \Pi) = \Pr(\Theta=1 \mid \sPi = \Pi) > \alpha+2\cdot \gamma^{1/2}
\end{align*}
by the assumption of $\ALG$. 
On the other hand, if the reduction stops by using $\frac{1}{5}\cdot \frac{\gamma^{3/2}}{\beta^2\alpha}$ arm pulls on the special arm, we show that the correct probability for the output of the \textbf{No} case is high. Note that if the instance is in the \textbf{Yes} case, the special arm is with mean reward $1/2$. Therefore, the arms for $\ALG$ are identical and independent random variables. Since the index of $\istar$ is chosen uniformly at random, by the same arguement we used in \Cref{lem:arm-trapping-new}, the expected number of arm pulls on the special arm is
\begin{align*}
\expect{s_{\istar}\mid \text{\textbf{Yes} case}} \leq \frac{1}{100}\cdot \frac{\gamma^{2}}{\alpha \cdot \beta^2}.
\end{align*}
As such, by a simple Markov bound, we have
\begin{align*}
\Pr\paren{s_{\istar}\geq \frac{1}{5}\cdot \frac{\gamma^{3/2}}{\beta^2\alpha}\mid \text{\textbf{Yes} case}}\leq \frac{\gamma^{1/2}}{20}.
\end{align*}
Therefore, the probability for Line~\ref{line:learn-break} to output correctly output \textbf{No} case is at least 
\begin{align*}
1-\frac{\gamma^{1/2}}{20} > 1-2 \, \gamma^{1/2}> 2 \, \gamma^{1/2},
\end{align*}
where the last inequality is by the range of $\gamma$. As such, the above implies
\begin{align*}
 \Pr_{\Pi}\paren{\Pr(\Theta=1 \mid \sPi = \Pi) =1}  > 2 \, \gamma^{1/2},
\end{align*}
and it forms the desired contradiction.
\end{proof}

\newcommand{\ftE}[3]{\ensuremath{f^{#1}_{\,#2}\paren{\,#3}}\xspace}
\newcommand{\Esmall}{\ensuremath{\mathcal{E}_{\Theta_{p}=0, s}}\xspace}
\newcommand{\Esample}{\ensuremath{\mathcal{E}_{s}}\xspace}
\newcommand{\Emem}[1]{\ensuremath{\mathcal{E}^{p}_{\text{mem}>1/2}(#1)}\xspace}
\newcommand{\Ememp}[1]{\ensuremath{\mathcal{E}_{\text{mem-obl}}^{#1}}\xspace}
\newcommand{\Ebatchp}[1]{\ensuremath{\mathcal{E}_{\text{batch-obl}}^{#1}}\xspace}
\newcommand{\statebefore}[1]{\ensuremath{\Pi^{\pi}_{>#1}, \, M^{m}_{>#1}}\xspace}
\newcommand{\pitilde}{\ensuremath{\tilde{\pi}}\xspace}
\newcommand{\Mtilde}{\ensuremath{\widetilde{M}}\xspace}
\newcommand{\samplep}{\ensuremath{\sample{\ALG}_{>p}}\xspace}
\newcommand{\SampEqual}{\ensuremath{\mathcal{S}_{\Theta_{p}=0}}\xspace}
\newcommand{\samplebatch}[1]{\ensuremath{\sample{\ALG}_{B_{#1}}}\xspace}
\newcommand{\samplepbefore}[2]{\ensuremath{\sample{\ALG}^{>#2}_{B_{#1}}}\xspace}
\newcommand{\Einform}[2]{\ensuremath{\mathcal{E}^{#1}_{\textnormal{inform}}(q) }\xspace}

\section{The Multi-Pass Lower Bound: Proof of \Cref{thm:main}}
\label{sec:main-lb}

We now proceed to the main part of the proof of \Cref{thm:main}. To continue, we introduce some additional notation used in the analysis in a self-contained manner. 

\paragraph{Additional Notation.} Let $P\geq 2$, $C \geq 1$ be postitive integers, $\cD(P,C)$ be the distribution of hard instances defined in  \Cref{sec:adversary}, and $\ALG$ be a $P$-pass (deterministic) streaming algorithm defined as in \Cref{subsec:model}.  
For each batch $B_{q}$, we use $\batchind{q}$ to denote the \emph{indices} in the \emph{order of the stream} of batch $q$, i.e. $\batchind{q} = ((P-q+1)\cdot \frac{n}{P+1}, (P-q+2)\cdot \frac{n}{P+1}]$. The batch $B_{q}$ hence contains the arms $\sigma(i)$ for $i\in \batchind{q}$.
For any integer $p \in [P+1]$, we introduce new notation to handle variables as functions of $p$ as follows.
\begin{itemize}
\item \textbf{Transcripts.} Denote $\sPi^{p}$ and $\Pi^p$ as the random variable and the realization of the transcript induced by the arm pulls \emph{within} the $p$-th pass. 
We further define 
$\sPi^{1:p} := (\sPi^1,\ldots,\sPi^p)$ and $\Pi^{1:p} := (\Pi^1,\ldots,\Pi^p)$ as the random variable and the realization of the transcript among \emph{all} of the first $p$-passes. For transcripts $\sPi^p$, $\sPi^{1:p}$, etc., we define batch-specific transcripts as follows. We define $\sPicap{q}^{p}$ (resp. $\Picap{q}^{p}$) be the transcript induced by the arm pulls \emph{on the arms in the $q$-th batch}, i.e. the result of $\arm_{\sigma(i)}$ is recorded in $\Picap{q}^{p}$ if $i \in \batchind{q}$. The notation generalizes to $\sPicap{q}^{1:p}$ as well.
%
\item \textbf{Memory.} We use $\sM^{p}$ and $M^p$ to denote the random variable and the realization of the memory state by the \emph{end} of the $p$-th pass. 
\item \textbf{Sample complexity.} We similary define $\sampleb{q}$ as the \emph{total} number of arm pulls used on the \emph{arms} from the $(q+1)$-th batch to the $(P+1)$-th batch, i.e. when calling the sampler $\cO$, the index $i \in \cup_{r=q+1}^{P+1} \batchind{r}$ (and it is independent of $\ja$). Similarly, we use $\samplebatch{q}$ to denote the \emph{total} number of arm pulls used on the arms in batch $q$. To avoid confusion, when we talk about the total number of arm pulls in pass $p$, it means the cumulative number of arm pulls in the \emph{first $p$ passes}. 
\end{itemize}
Notice that the final output of the algorithm is a deterministic function of $(\Pi^{1:P},M^{P})$. 

Since we work with the expectation over the randomness of the memory and transcript, to avoid very long lines, we sometimes use
\[\Exp_{\sPi^{1:p-1}, \sM^{p-1}}\bracket{\sample{\ALG} \mid \Pi^{1:p-1}, M^{p-1}, \Theta_{<p}=0} \]  
for the full expression of $\Exp_{\sPi^{1:p-1}, \sM^{p-1}}\bracket{\sample{\ALG} \mid \sPi^{1:p-1}=\Pi^{1:p-1}, \sM^{p-1}=M^{p-1}, \Theta_{<p}=0}$. Other random variables that appear on the conditions follow the same rule for simplifications. 


\paragraph{Memory- and Batch-obliviousness.} We now introduce the notion of memory- and batch-obliviousness, which we will use crucially to describe the ``limits of learning'' for any streaming algorithms.

We say that the algorithm is \textbf{memory-oblivious} at the end of pass $p$ if $M^{p}$ contains no arm with reward strictly more than $1/2$ during \emph{any} of the first $p$ passes. Notice that in particular if the algorithm is memory-oblivious 
at the end of the $P$-th pass, then it cannot output the best arm in the stream. We use $\Ememp{p}$ to denote the \emph{event} that $\ALG$ is memory-oblivious by the \emph{end} of pass $p$. Note that the memory-oblivious event has \emph{downward implications}: if the algorithm is memory oblivious at the end of pass $p$, it has to be memory-oblivious for all passes $p'<p$. 

We further say that the algorithm is \textbf{batch-oblivious} at the end of pass $p$ if given $(\Pi^{1:p}, M^{p})$, and conditioning on the event of $\Theta_{<p} = 0$, for any $q > p$, the algorithm ``does not know'' the value of $\Theta_q$; formally, 
\begin{align}
	\forall~p < q \leq P: \quad \Pr\paren{\Theta_q = 1 \mid \sPi^{1:p}=\Pi^{1:p},\sM^{p}=M^{p}, \Theta_{<p+1} = 0} \in [\frac{1}{2P} - \frac{1}{4 P^2}\,,\, \frac{1}{2P} + \frac{1}{4P^2}]. \label{eq:batch-oblivious} 
\end{align}
We use $\Ebatchp{p}$ to denote the \emph{event} that $\ALG$ is batch-oblivous by the \emph{end} of pass $p$.

\medskip

\paragraph{The strategy of the proof.} We will inductively show that the algorithm is going to be memory-oblivious and batch-oblivious with a large probability throughout each passes. 
To do so, we consider two types of possible behavior for the algorithm $\ALG$ in each pass $p$: 
\begin{itemize}
	\item \textbf{Conservative} case: the first case is when the algorithm decides to be ``conservative'' with its arm pulls in the first $p$ passes.  
	We show that if the probability for the algorithm to be in conservative case after the first $p-1$ passes is large, and the algorithm decides to be conservative on the first $p$ passes, then the algorithm is going to remain memory-oblivious and batch-oblivious for the subsequent pass as well with a sufficiently large probability. 
	
	\item \textbf{Radical} case: the complementary case is when the algorithm decides to make ``many'' arm pulls in the first $p$ passes. In this case, we use the memory- and batch-obliviousness properties of the algorithm
	to show that such an algorithm is necessarily going to break the guarantee on the number of arm pulls imposed on it by \Cref{thm:main} in some cases.  
\end{itemize}

Formalizing this strategy is challenging due to the nature of guarantee of \Cref{thm:main} on the event $\eventF(p)$ for some unknown $p$ (rather informally speaking, since $\eta_{p}$ is unknown to the algorithm, but also follows a certain distribution in the input). This requires a careful conditioning on various events happening in the algorithm and keeping track of the information revealed by these events. 


\subsection{The Conservative Case} 
\label{subsec:conservative-case}
The following lemma allows us to handle the conservative case. 

\begin{lemma}[Conservative case]
\label{lem:nfl-no-learn}
For any integer $p \in [P]$, let $\ALG$ be a streaming algorithm with a memory of at most ${n}/{(20000 P^3)}$ arms, 
and assume that at the end of the pass $p-1$, the following conditions hold
\begin{enumerate}[label=(\Roman*).]
\item\label{line:induc-prev-prob} The probability for $\Ebatchp{p-1}, \Ememp{p-1}$ and $\Theta_{<p}=0$ to happen is large, i.e.,
\[\Pr\paren{\Ebatchp{p-1}, \Ememp{p-1}, \Theta_{<p} = 0} \geq \paren{1-\frac{1}{2P}}^{10(p-1)}; \]
\item\label{line:small-exp-sample} Conditioning on $\Theta_{<p} = 0$ and $\Ebatchp{p-1}, \Ememp{p-1}$, the expected number of arm pulls (over the randomness of the first $p$ passes) is small, i.e.,
\[
	\Exp\bracket{\sampleb{p} \mid \Ebatchp{p-1}, \Ememp{p-1}, \Theta_{<p}=0} \leq \frac{1}{10^9}\cdot \frac{n}{\gamma_{p+1} \cdot P^{30}}.
\]
\end{enumerate}
Then, with probability at least $(1-1/2P)^{10p}$, at the end of pass $p$, we have $\Theta_{<p+1} = 0$ and the algorithm is memory- and batch-oblivious, i.e.,
\[\Pr\paren{\Ebatchp{p}, \Ememp{p}, \Theta_{<p+1}=0} \geq \paren{1-\frac{1}{2P}}^{10p}.\] 
\end{lemma}

We prove \Cref{lem:nfl-no-learn} in the rest of \Cref{subsec:conservative-case}. To this end, we show two main technical lemmas towards the proof of the memory and batch obliviousness. 

The first technical lemma characterizes a ``no storing'' constraint of a $p$-pass algorithm that satisfies the conditions as prescribed in \Cref{lem:nfl-no-learn}. This is a ``streaming version'' of the offline ``no trapping'' result of \Cref{lem:arm-trapping-new}.
\begin{lemma}
\label{lem:batch-storing-limit}
Let $p\in[P+1]$ be a parameter, and let $\ALG$ be a $p$-pass streaming algorithm with a memory of $n/(20000 P^3)$ arms. Let $q\in (p, P+1]$,
and suppose the underlying instance from $\cD(P,C)$ satisfies that 
batches $B_{\leq p}$ contain only arms with mean rewards $\frac{1}{2}$, i.e. $\Theta_{\leq p}=0$.

Furthermore, suppose the assumptions of \Cref{line:induc-prev-prob} and \Cref{line:small-exp-sample} in \Cref{lem:nfl-no-learn} hold. Then, \emph{conditioning on} $\Ebatchp{p-1}, \Ememp{p-1}, \Theta_{\leq p} = 0$, the probability for $\ALG$ to store an arm with mean reward \emph{strictly more than} $\frac{1}{2}$ from $B_{q}$ is at most $\frac{1}{2P^2}$, i.e., let $\Emem{q}$ be the event that $\ALG$ stores the special arm of batch $q$, there is
\begin{align*}
	& \Pr\paren{\Emem{q}\mid \Ebatchp{p-1}, \Ememp{p-1}, \Theta_{\leq p} = 0} \\
	& = \Exp_{\substack{\sPi^{1:p}, \sM^{p}}}\bracket{\Pr\paren{\Emem{q}\mid \Ebatchp{p-1}, \Ememp{p-1}, \Theta_{\leq p} = 0, \sPi^{1:p}=\Pi^{1:p}, \sM^{p}=M^{p}}} \leq \frac{1}{2P^2}.
\end{align*}
We explicitly write the expectation over $\sPi^{1:p}, \sM^{p}$ to emphasize the randomness over the transcript and the memory of the first $p$ passes.
\end{lemma}


\begin{proof}
We first note that the randomness in the statement of \Cref{lem:batch-storing-limit} includes the randomness of the transcript and the memory of the first $p$ passes and the underlying instance. 
We use a reduction argument from the problem in \Cref{lem:arm-trapping-new} to establish the desired lower bound. In particular, we show that conditioning on all the conditions in \Cref{lem:batch-storing-limit}, if $\ALG$ can store the special arm in batch $B_{q}$ with probability more than $1/2P^2$, then we can ``trap'' the best arm in batch $q$ by running $\ALG$ with $\gamma = \frac{1}{100 P^2}$, $k=\frac{n}{P+1}$, and $\beta=\eta_{q}$. The success probability is non-negligible, albeit low, and such an algorithm would require a high sample complexity. However, since we assume low sample complexity (condition \Cref{line:small-exp-sample}), it will lead to a contradiction with \Cref{lem:arm-trapping-new}.

We now formalize the above intuition. We first give a detailed simulation procedure as follows.  

\begin{tbox}
	\textbf{An algorithm (reduction) for the problem in \Cref{lem:arm-trapping-new}}
	
	\smallskip
	
	\textbf{Input:} $B_{q}$: $k$ arms with one \emph{special arm} as in \Cref{lem:arm-trapping-new} with $\beta=\eta_{q}$; \\
	\textbf{Input:} $\ALG$: a streaming algorithm that stores the special arm of $B_{q}$ with probability more than $1/2P^2$ conditioning on the event of \Cref{lem:batch-storing-limit}.\\ 
	\textbf{Parameters:} $\gamma = \frac{1}{100 P^2} \qquad k=\frac{n}{P+1} \qquad \beta=\eta_{q}$. 
	
	\smallskip
	
	\begin{enumerate}
		\item Sample an underlying instance from $\cD(P,C)$ for $\ALG$ as follows
		\begin{enumerate}
			\item Parameters $\eta_{r}$ in $\cD(P,C)$ as follows: let $(P-q+1)$ parameters $\eta_{r}$ follow the arriving order before $B_{q}$, and let $(q-1)$ parameters $\eta_{r}$ follow the arriving order after $B_{q}$. 
			\item Ensure the condition of $\Theta_{\leq p}=0$, and sample each $\Theta_{r}=1$ for $r\not\in [p]\cup \{q\}$ with probability $1/2P$ (exactly as in $\cD(P,C)$).
			\item Sample $P$ batches of $\frac{n}{P+1}$ arms with the above setting, and concatenate them with $B_q$ to get the stream.
		\end{enumerate}
		\item Run the streaming algorithm $\ALG$ on the instance:
		\begin{enumerate}
			\item For each pass, sample exactly as $\ALG$ does and maintain the local memory exactly as the memory of $\ALG$.
			\item At any point, if the number of samples is more than $\frac{1}{25000}\cdot \frac{n}{\gamma_{p+1} \cdot P^{29}}$ on batch $q$, abort the algorithm and output ``failure''.
			\item If the algorithm does \emph{not} output ``failure'', at the end of the $p$-th pass, output all the (indices of) arms that are in $B_{q}$.
		\end{enumerate} 
	\end{enumerate}
\end{tbox}

As we have discussed in \Cref{subsec:model}, the offline algorithm can ignore the indices of the pass and the arriving arm in $\Pi$ (by writing $*$ therein). As such, the reduction gives a valid algorithm for the problem in \Cref{lem:arm-trapping-new}. 
We now lower bound the probability of $(\Emem{q}\mid \Theta_{\leq p}=0)$ using the probability of $\Emem{q}$ \emph{conditioning on} \Cref{line:induc-prev-prob} of \Cref{lem:nfl-no-learn}. Formally, we have
\begin{align*}
 	& \Pr\paren{\Emem{q} \mid \Theta_{\leq p}=0} \tag{written in the conditional form to begin with by the conditions in \Cref{lem:arm-trapping-new}}\\
	 & \geq \Pr\paren{\Emem{q} \mid \Ebatchp{p-1}, \Ememp{p-1}, \Theta_{\leq p}=0} \cdot \Pr\paren{\Ebatchp{p-1}, \Ememp{p-1} \mid \Theta_{\leq p}=0}.
\end{align*}
We lower bound the second term by using the condition in \Cref{line:induc-prev-prob} of \Cref{lem:nfl-no-learn} as follows:
\begin{align*}
	\Pr\paren{\Ebatchp{p-1}, \Ememp{p-1} \mid \Theta_{\leq p}=0}
	&= \frac{\Pr\paren{\Ebatchp{p-1}, \Ememp{p-1}, \Theta_{\leq p}=0}}{\Pr\paren{\Theta_{\leq p}=0}}\\
	&\geq \Pr\paren{\Ebatchp{p-1}, \Ememp{p-1}, \Theta_{\leq p}=0}\\
	&=\Pr\paren{\Ebatchp{p-1}, \Ememp{p-1}, \Theta_{< p}=0, \Theta_{p}=0}\\
	&= \Pr\paren{\Theta_{p}=0\mid \Ebatchp{p-1}, \Ememp{p-1}}\cdot \Pr\paren{\Ebatchp{p-1}, \Ememp{p-1}, \Theta_{< p}=0}\\
	& \geq \Pr\paren{\Theta_{p}=0\mid \Ebatchp{p-1}, \Ememp{p-1}}\cdot \paren{1-\frac{1}{2P}}^{10(P-1)} \tag{by the condition of \Cref{line:induc-prev-prob}}\\
	&\geq \paren{1-\frac{1}{4P}}\cdot \paren{1-\frac{1}{2P}}^{10(P-1)} \tag{by using the batch obliviousness}\\
	&\geq \frac{1}{30}. \tag{using $P\geq 2$}
\end{align*}
On the other hand, recall that by condition \Cref{line:small-exp-sample} of \Cref{lem:nfl-no-learn}, there is
\begin{align*}
	\expect{\sampleb{p} \mid \Ebatchp{p-1}, \Ememp{p-1}, \Theta_{< p}=0} \leq \frac{1}{10^9}\cdot \frac{n}{\gamma_{p+1} \cdot P^{30}}.
\end{align*}
We bound the expected sample of $(\sampleb{p} \mid \Ebatchp{p-1}, \Ememp{p-1}, \Theta_{\leq p}=0)$ (note the extra condition of $\Theta_{p}=0$) with the batch oblivious condition:
\begin{align*}
	& \expect{\sampleb{p} \mid \Ebatchp{p-1}, \Ememp{p-1}, \Theta_{\leq p}=0}\\
	& = \expect{\sampleb{p} \mid \Ebatchp{p-1}, \Ememp{p-1}, \Theta_{p}=0, \Theta_{< p}=0} \\
	& \leq \frac{1}{\Pr\paren{\Theta_{p}=0\mid \Ebatchp{p-1}, \Ememp{p-1}, \Theta_{< p}=0}}\cdot \expect{\sampleb{p} \mid \Ebatchp{p-1}, \Ememp{p-1}, \Theta_{< p}=0} \tag{by total expectation}\\
	&\leq \frac{1}{1-3/4P}\cdot \expect{\sampleb{p} \mid \Ebatchp{p-1}, \Ememp{p-1}, \Theta_{< p}=0} \tag{by batch obliviousness}\\
	&\leq 2\cdot \expect{\sampleb{p} \mid \Ebatchp{p-1}, \Ememp{p-1}, \Theta_{< p}=0}\\
	&\leq \frac{2}{10^9}\cdot \frac{n}{\gamma_{p+1} \cdot P^{30}}.
\end{align*} 
Note additionally that the total arm pulls we used on batch $q$ is a subset of the arm pulls measured by $\sampleb{p}$. Therefore, conditioning on $(\Ebatchp{p-1}, \Ememp{p-1}, \Theta_{q}=1 ,\Theta_{\leq p}=0)$, with probability at least $(1-\frac{1}{100P})$, the sample complexity does not break the limit and return ``failure'' (Markov bound).

We further note that our input instance has $\Theta_{q}=1$ deterministically. As such, we can further lower bound the probability our reduction to store the best arm by
\begin{align*}
	& \Pr\paren{\text{ALG stores the special arm}\mid \Theta_{q}=1,\Theta_{\leq p}=0} \\
	&= \frac{\Pr\paren{\text{ALG stores the special arm}\mid \Theta_{\leq p}=0}}{\Pr\paren{\Theta_{q}=1\mid \Theta_{\leq p}=0}} \tag{$\Pr(\text{ALG stores the special arm}\mid \Theta_{q}=0, \Theta_{\leq p}=0)=0$}\\
	&\geq \Pr\paren{\text{ALG stores the special arm}\mid \Ebatchp{p-1}, \Ememp{p-1}, \Theta_{\leq p}=0} \cdot \Pr\paren{\Ebatchp{p-1}, \Ememp{p-1} \mid \Theta_{\leq p}=0}\\
	&\geq \Pr\paren{\Emem{q} \mid \Ebatchp{p-1}, \Ememp{p-1}, \Theta_{\leq p}=0}\cdot \Pr\paren{\text{not return ``failure''}\mid \Ebatchp{p-1}, \Ememp{p-1}, \Theta_{\leq p}=0}\\
	& \qquad \cdot \Pr\paren{\Ebatchp{p-1}, \Ememp{p-1} \mid \Theta_{\leq p}=0}\\
	& \geq \Pr\paren{\Emem{q} \mid \Ebatchp{p-1}, \Ememp{p-1}, \Theta_{\leq p}=0} \cdot \paren{1-\frac{1}{100P}}\cdot \frac{1}{30}.
\end{align*}
Hence, the condition of 
\[\Pr\paren{\Emem{q} \mid \Ebatchp{p-1}, \Ememp{p-1}, \Theta_{\leq p}=0}>\frac{1}{2P^2}\]
implies
\begin{align*}
		\Pr\paren{\text{ALG stores the special arm}\mid \Theta_{\leq p} = 0} \geq (1-\frac{1}{100P})\cdot \frac{1}{30}\cdot \frac{1}{2P^2} \geq \frac{1}{100 P^2}.
\end{align*}
The final output of the reduction is a subset of the memory of the streaming algorithm $\ALG$, which is at most $\frac{n}{20000 P^3}$. Note that for any $P\geq 2$, there is $300\cdot (100 P^2)^3 \cdot (P+1)< 25000 \cdot P^{29}$. Therefore, we obtain an offline algorithm that uses at most $\frac{1}{25000}\cdot \frac{n}{\eta^2_{q}}\cdot \frac{1}{P^{29}}<\frac{1}{300}\cdot \frac{\gamma^3}{\beta^2}\cdot k$ arm pulls and outputs at most $\frac{n}{20000 P^3}<\frac{1}{12}\cdot \frac{1}{500 P^2}\cdot k$ arms (using $P\geq 2$) that contains the special arm with probability at least $\gamma=\frac{1}{100 P^2}$. This reaches a contradiction with \Cref{lem:arm-trapping-new}, which proves \Cref{lem:batch-storing-limit}.
\end{proof}

We now show another technical lemma that deals with the ``learning'' aspect of a $p$-pass streaming algorithm that satisfies the conditions in \Cref{lem:nfl-no-learn}. This is similarly a streaming analogy of the offline ``no learning'' result of \Cref{lem:sample-success-tradeoff}.

\begin{lemma}
	\label{lem:batch-learning-limit}
	Let $p\in[P+1]$ be a parameter, and let $\ALG$ be a $p$-pass streaming algorithm with a memory of $n/(20000 P^3)$ arms. 
	Suppose the underlying instance from $\cD(P,C)$ satisfies that batches $B_{\leq p}$ contain only arms with mean rewards $\frac{1}{2}$, i.e. $\Theta_{\leq p}=0$.
	Additionally, suppose the conditions of \Cref{line:induc-prev-prob} and \Cref{line:small-exp-sample} in \Cref{lem:nfl-no-learn} hold, and there is 
	\begin{align*}
		\Pr\paren{\Ememp{p}, \Ebatchp{p-1}, \Theta_{\leq p}=0}\geq \paren{1-\frac{1}{2P}}^{10(p-1)+5}.
	\end{align*}
	Then, for any $q\in (p, P]$, with probability at least $(1-\frac{1}{2P^2})$ \emph{conditioning on} $(\Ememp{p}, \Ebatchp{p-1}, \Theta_{\leq p}=0)$, there is 
	\begin{align*}
		\Pr\paren{\Theta_{q}=1\mid \sPi^{1:p}=\Pi^{1:p}, \sM^{p}=M^{p}, \Theta_{< p+1}=0} \in [\frac{1}{2P} - \frac{1}{4 P^2}\,,\, \frac{1}{2P} + \frac{1}{4P^2}].
	\end{align*}
\end{lemma}
\begin{proof}
We only show the proof for the upper bound since the lower bound follows in the same manner.
Similar to the proof of \Cref{lem:batch-storing-limit}, we show an algorithm (reduction) from the offline \Cref{lem:sample-success-tradeoff} to the streaming algorithm as follows. 
\begin{tbox}
	\textbf{An algorithm (reduction) for the problem in \Cref{lem:sample-success-tradeoff}}
	
	\smallskip
	
	\textbf{Input:} $B_{q}$: $k$ arms following the distribution of \Cref{lem:sample-success-tradeoff} with $\beta=\eta_{q}$; \\
	\textbf{Input:} $\ALG$: a streaming algorithm that with probability more than $1/2P^2$ conditioning on the conditions of \Cref{lem:batch-learning-limit}, outputs memory and transcript $\Pi^{1:p}$ and $M^{p}$ such that $\Pr(\Theta_{q}=1\mid \sPi^{1:p}=\Pi^{1:p}, \sM^{p}=M^{p}, \Theta_{< p+1}=0) > \frac{1}{2P} + \frac{1}{4P^2}$. \\ 
	\textbf{Parameters:} $\gamma^{1/2} = \frac{1}{250 P^2} \qquad k=\frac{n}{P+1} \qquad \beta=\eta_{q}$. 
	
	\smallskip
	
	\begin{enumerate}
		\item Sample an underlying instance from $\cD(P,C)$ for $\ALG$ as follows
		\begin{enumerate}
			\item Parameters $\eta_{r}$ in $\cD(P,C)$ as follows: let $(P-q+1)$ parameters $\eta_{r}$ follwing the arriving order before $B_{q}$, and let $(q-1)$ parameters $\eta_{r}$ follwing the arriving order after $B_{q}$. 
			\item Ensure the condition of $\Theta_{\leq p}=0$, and sample each $\Theta_{r}=1$ for $r\not\in [p]\cup \{q\}$ with probability $1/2P$ (exactly as in $\cD(P,C)$ for these batches).
			\item Sample $P$ batches of $\frac{n}{P+1}$ arms with the above setting, and concatenate them with $B_q$ to get the stream.
		\end{enumerate}
		\item Run the streaming algorithm $\ALG$ on the instance:
		\begin{enumerate}
			\item For each pass, sample exactly as $\ALG$ does and maintain the local memory exactly as the memory of $\ALG$.
			\item At any point, if the number of samples is more than $\frac{1}{5}\cdot \frac{1}{10^5}\cdot \frac{n}{\gamma_{p+1} \cdot P^{30}}$ on batch $q$, abort the algorithm and output ``failure''.
			\item If the algorithm does \emph{not} output ``failure'', at the end of the $p$-th pass, evaluate $\Theta_{q}$ with maximum likelihood estimation, i.e., let the trasncript and memory of the algorithm be $(\Pi^{1:p}, M^{p})$, we output the distribution of 
			\[\paren{\Theta_{q}\mid \sPi^{1:p}=\Pi^{1:p}, \sM^{p}=M^{p}, \Theta_{\leq p}=0}.\]
		\end{enumerate} 
	\end{enumerate}
\end{tbox}
We show that with probability more than $2\gamma^{1/2}$, the algorithm returns $\Pr(\Theta=1\mid \sPi=\Pi)>\frac{1}{2P} + 2\gamma^{1/2}$. Define a pair of transcript and memory $\sPi^{1:p}=\Pi^{1:p}, \sM^{p}=M^{p}$ as \emph{informative} if $\Pr(\Theta_{q}=1\mid \sPi^{1:p}=\Pi^{1:p}, \sM^{p}=M^{p}, \Theta_{< p+1}=0)>\frac{1}{2P} + 2\gamma^{1/2}$, and define $\Einform{p}{q}$ as the \emph{event} for the streaming algorithm to produce informative transcript and memory by the end of pass $p$ for batch $q$. We first lower bound the probability of the event as follows. 
\begin{align*}
& \Pr\paren{\Einform{p}{q}\mid \Theta_{< p+1}=0}\\
& \geq \Pr\paren{\Einform{p}{q}\mid \Ememp{p}, \Ebatchp{p-1}, \Theta_{< p+1}=0}\cdot \Pr\paren{\Ememp{p}, \Ebatchp{p-1}\mid \Theta_{< p+1}=0}\\
& \geq \Pr\paren{\Einform{p}{q}\mid \Ememp{p}, \Ebatchp{p-1}, \Theta_{< p+1}=0}\cdot \Pr\paren{\Ememp{p}, \Ebatchp{p-1}, \Theta_{< p+1}=0}\\
&\geq \Pr\paren{\Einform{p}{q}\mid \Ememp{p}, \Ebatchp{p-1}, \Theta_{< p+1}=0} \cdot \paren{1-\frac{1}{2P}}^{10(p-1)+5}\\
&\geq \Pr\paren{\Einform{p}{q}\mid \Ememp{p}, \Ebatchp{p-1}, \Theta_{< p+1}=0} \cdot \frac{1}{200},
\end{align*}
where the second-last line is from the condition of \Cref{lem:batch-learning-limit}, and the last line uses $p\leq P$ and $P\geq 2$. We now provide an upper bound on the expected number of arm pulls using the condition of \Cref{line:small-exp-sample} in \Cref{lem:nfl-no-learn} with the \emph{extra condition} of $\Ememp{p}$. To this end, we first upper bound $\Pr(\Ememp{p}, \Ebatchp{p-1}, \Theta_{\leq p}=0)$ with the term $\Pr(\Ememp{p}, \Theta_{p}=0 \mid \Ebatchp{p-1}, \Ememp{p-1}, \Theta_{<p}=0)$ as follows.
\begin{align*}
	&\Pr\paren{\Ememp{p}, \Ebatchp{p-1}, \Theta_{\leq p}=0}\\
	&= \Pr\paren{\Ememp{p}, \Ebatchp{p-1}, \Theta_{\leq p}=0\mid \Ememp{p-1}}\cdot \Pr\paren{\Ememp{p-1}} \tag{$\Ememp{p}$ cannot happen if $\Ememp{p-1}$ does \emph{not} happen}\\
	&\leq \Pr\paren{\Ememp{p}, \Ebatchp{p-1}, \Theta_{\leq p}=0\mid \Ememp{p-1}}\\
	&= \Pr\paren{\Ememp{p}, \Ebatchp{p-1}, \Theta_{<p}=0, \Theta_{p}=0\mid \Ememp{p-1}}\\
	&= \Pr\paren{\Ememp{p}, \Theta_{p}=0 \mid \Ebatchp{p-1}, \Ememp{p-1}, \Theta_{<p}=0}\cdot \Pr\paren{\Ebatchp{p-1}, \Theta_{<p}=0\mid \Ememp{p-1}}\\
	&\leq \Pr\paren{\Ememp{p}, \Theta_{p}=0 \mid \Ebatchp{p-1}, \Ememp{p-1}, \Theta_{<p}=0}.
\end{align*}
With the above inequality, we can bound the expected samples on $\sampleb{p}$ as follows.
\begin{align*}
	& \expect{\sampleb{p} \mid \Ebatchp{p-1}, \Ememp{p-1}, \Ememp{p}, \Theta_{\leq p}=0}\\
	& \leq \expect{\sampleb{p} \mid \Ebatchp{p-1}, \Ememp{p-1}, \Theta_{<p}=0}\cdot \frac{1}{\Pr\paren{\Ememp{p}, \Theta_{p}=0 \mid \Ebatchp{p-1}, \Ememp{p-1}, \Theta_{<p}=0}}\\
	&\leq \expect{\sampleb{p} \mid \Ebatchp{p-1}, \Ememp{p-1}, \Theta_{<p}=0}\cdot \frac{1}{\Pr\paren{\Ememp{p}, \Ebatchp{p-1}, \Theta_{\leq p}=0}}\\
	& \leq \expect{\sampleb{p} \mid \Ebatchp{p-1}, \Ememp{p-1}, \Theta_{<p}=0}\cdot \frac{1}{\paren{1-\frac{1}{2P}}^{10(p-1)+5}} \tag{by the condition in \Cref{lem:batch-learning-limit}}\\
	&\leq \expect{\sampleb{p} \mid \Ebatchp{p-1}, \Ememp{p-1}, \Theta_{<p}=0}\cdot 200 \\
	&\leq \frac{1}{5}\cdot \frac{1}{10^6}\cdot \frac{n}{\gamma_{p+1} \cdot P^{30}}. \tag{using \Cref{line:small-exp-sample} in \Cref{lem:nfl-no-learn}}
\end{align*}
Once again, note that the total arm pulls we used on batch $q$ is a subset of the arm pulls measured by $\sampleb{p}$. Therefore, by the Markov bound, with probability at least $9/10$ \emph{conditioning on} the events of $\Ebatchp{p-1}, \Ememp{p-1}, \Ememp{p}, \Theta_{\leq p}=0$, the algorithm will \emph{not} return failure. 

Note that by running the reduction, the offline algorithm has access of $\Pi^{1:p}$ and $M^{P}$ by the end of pass $p$. 
Hence, if we have 
\[\Pr(\Theta_{q}=1\mid \sPi^{1:p}=\Pi^{1:p}, \sM^{p}=M^{p}, \Theta_{< p+1}=0)>\frac{1}{2P} + 2\gamma^{1/2},\] 
there is 
\[\Pr(\Theta=1\mid \sPi=\Pi)> \frac{1}{2P} + 2\gamma^{1/2}\] 
from the perspective of the offline algorithm. Therefore, we have that
\begin{align*}
	\Pr\paren{\Pr(\Theta=1\mid \sPi=\Pi)> \frac{1}{2P} + 2\gamma^{1/2}}&= \Pr\paren{\Pr(\Theta=1\mid \sPi=\Pi)> \frac{1}{2P} + 2\gamma^{1/2}\mid \Theta_{<{p+1}}=0} \tag{$\Theta_{<{p+1}}=0$ is ensured in the instances}\\
	& = \Pr\paren{\Einform{p}{q} \mid \Theta_{<{p+1}}=0}.
\end{align*}
As such, we can combine this with the lower bound of $\Pr(\Einform{p}{q}\mid \Theta_{< p+1}=0)$ to get that if we have 
\[\Pr\paren{\Einform{p}{q}\mid \Ememp{p}, \Ebatchp{p-1}, \Theta_{< p+1}=0} \geq \frac{1}{2P^2},\] 
it implies that
\begin{align*}
	\Pr\paren{\Pr(\Theta=1\mid \sPi=\Pi)> \frac{1}{2P} + 2\gamma^{1/2}} &=
	\Pr\paren{\Einform{p}{q}\mid \Theta_{< p+1}=0} \\
	& \geq \Pr\paren{\Einform{p}{q}\mid \Ememp{p}, \Ebatchp{p-1}, \Theta_{< p+1}=0}\cdot \frac{1}{200}\\
	&\geq \frac{1}{2P^2} \cdot \frac{1}{200}\cdot \frac{9}{10} \geq \frac{1}{500P^2}.
\end{align*}
Furthermore, the algorithm uses at most $\frac{1}{5}\cdot \frac{1}{10^5}\cdot \frac{n}{\gamma_{p+1} \cdot P^{30}}$ arm pulls on batch $B_{q}$.

Let $2\gamma^{1/2}=\frac{1}{500P^2}$, we have the guarantee of the knowledge on $\Theta$ (of the offline algorithm) becomes $\Pr(\Theta=1\mid \sPi=\Pi)>\frac{1}{2P} + \frac{1}{4P^2}> \frac{1}{2P} + \frac{1}{500 P^2}$. Therefore, such a algorithm should require $\frac{1}{100}\cdot \frac{\gamma^{2} \cdot k}{\alpha \cdot \beta^2}$ arm pulls. For any $P\geq 2$, there is $\frac{100\cdot (250 P^2)^4\cdot (P+1)}{2P}< 5\cdot 10^5\cdot P^{30}$, which implies $\frac{1}{5}\cdot \frac{1}{10^5}\cdot \frac{n}{\gamma_{p+1} \cdot P^{30}}< \frac{1}{100}\cdot \frac{\gamma^{2} \cdot k}{\alpha \cdot \beta^2}$ arm pulls ($\beta=\eta_{q}\leq \sqrt{\gamma_{p+1}}$). This forms a contradiction with \Cref{lem:sample-success-tradeoff}, which proves the lemma.
\end{proof}

We are ready proceed to the proof of the main claims in \Cref{lem:nfl-no-learn} as follows. 

\begin{proof}[\textbf{Proof of \Cref{lem:nfl-no-learn}}]
We first lower bound the probability of $\Pr(\Ebatchp{p}, \Ememp{p}, \Theta_{<p+1}=0)$ as a function of $\Pr(\Ebatchp{p}, \Ememp{p}\mid \Ebatchp{p-1}, \Ememp{p-1}, \Theta_{<p+1}=0)$. To this end, we lower bound $\Pr(\Ebatchp{p}, \Ememp{p}, \Theta_{<p+1}=0)$ as follows:
\begin{align*}
	& \Pr\paren{\Ebatchp{p}, \Ememp{p}, \Theta_{<p+1}=0} \\
	& \geq \Pr\paren{\Ebatchp{p}, \Ememp{p}, \Theta_{<p+1}=0 \mid \Ebatchp{p-1}, \Ememp{p-1}, \Theta_{<p}=0}\cdot \Pr\paren{\Ebatchp{p-1}, \Ememp{p-1}, \Theta_{<p}=0},
\end{align*}
in which the first term of the right hand side can be factored to
\begin{align*}
	& \Pr\paren{\Ebatchp{p}, \Ememp{p}, \Theta_{<p+1}=0 \mid \Ebatchp{p-1}, \Ememp{p-1}, \Theta_{<p}=0} \\
	&=\Pr\paren{\Theta_{<p+1}=0\mid \Ebatchp{p-1}, \Ememp{p-1}, \Theta_{<p} = 0}\cdot \Pr\paren{\Ebatchp{p}, \Ememp{p}\mid \Ebatchp{p-1}, \Ememp{p-1}, \Theta_{<p+1} = 0}\\
	&=  \Pr\paren{\Theta_{p}=0\mid \Ebatchp{p-1}, \Ememp{p-1}, \Theta_{<p} = 0}\cdot \Pr\paren{\Ebatchp{p}, \Ememp{p}\mid \Ebatchp{p-1}, \Ememp{p-1}, \Theta_{<p+1} = 0} \tag{expanding the condition of $\Theta_{<p+1}=0$}\\
	&\geq \paren{1-\frac{1}{2P}-\frac{1}{4P^2}}\cdot \Pr\paren{\Ebatchp{p}, \Ememp{p}\mid \Ebatchp{p-1}, \Ememp{p-1}, \Theta_{<p+1} = 0}  \tag{by batch-obliviousness by the end of pass $(p-1)$}\\
	&\geq (1-\frac{3}{4P})\cdot \Pr\paren{\Ebatchp{p}, \Ememp{p}\mid \Ebatchp{p-1}, \Ememp{p-1}, \Theta_{<p+1} = 0}.
\end{align*}
Therefore, we obtain a valid lower bound for $\Pr\paren{\Ebatchp{p}, \Ememp{p}, \Theta_{<p+1}=0}$ as 
\begin{align*}
	& \Pr\paren{\Ebatchp{p}, \Ememp{p}, \Theta_{<p+1}=0}\\
	&\geq (1-\frac{3}{4P})\cdot \Pr\paren{\Ebatchp{p-1}, \Ememp{p-1}, \Theta_{<p}=0} \cdot  \Pr\paren{\Ebatchp{p}, \Ememp{p}\mid \Ebatchp{p-1}, \Ememp{p-1}, \Theta_{<p+1} = 0}\\
	&\geq (1-\frac{3}{4P}) \cdot \paren{1-\frac{1}{2P}}^{10(p-1)}\cdot \Pr\paren{\Ebatchp{p}, \Ememp{p}\mid \Ebatchp{p-1}, \Ememp{p-1}, \Theta_{<p+1} = 0}.
\end{align*}
As such, we only need to lower boud the last term, i.e., the term of $\Pr(\Ebatchp{p}, \Ememp{p}\mid \Ebatchp{p-1}, \Ememp{p-1}, \Theta_{<p+1} = 0)$. To this end, we further write the probability as
\begin{align*}
	& \Pr\paren{\Ebatchp{p}, \Ememp{p}\mid \Ebatchp{p-1}, \Ememp{p-1}, \Theta_{<p+1} = 0} \\
	&= \Pr\paren{\Ebatchp{p}\mid \Ememp{p}, \Ebatchp{p-1}, \Ememp{p-1}, \Theta_{<p+1}=0}\cdot \Pr\paren{\Ememp{p}\mid \Ebatchp{p-1}, \Ememp{p-1}, \Theta_{<p+1}=0}.
\end{align*}
We start with bounding the term $\Pr\paren{\Ememp{p}\mid \Ebatchp{p-1}, \Ememp{p-1}, \Theta_{<p+1}=0}$, i.e., the memory-oblivious proof.
\paragraph{Memory oblivious proof.} For each batch $q\in [P+1]$, we define the following event.
\begin{center}
	$\Emem{q}$: the event that $\memory{\ALG}$ contains an arm with mean reward more than $\frac{1}{2}$ from arms in batch $B_q$.
\end{center}
Recall that $\Ememp{p}$ is the \emph{event} that $\ALG$ is memory-oblivious \emph{after} pass $p$. By a simple union bound, we can bound the probability for $\Ememp{p}$ \emph{not} to happen, i.e., the memory contains at least one arm with reward strictly more than $\frac{1}{2}$, as follows.
\begin{align*}
	& \Pr\paren{\neg\, \Ememp{p} \mid \Ebatchp{p-1}, \Ememp{p-1}, \Theta_{<p+1}=0} \\
	& \leq \sum_{q\in [P+1]} \Pr\paren{\Emem{q} \mid \Ebatchp{p-1}, \Ememp{p-1}, \Theta_{<p+1}=0}.
\end{align*}
It suffices to upper bound each conditional probability term of $\Emem{q}$. To this end, we bound terms for $q$ of different types.

\noindent
\textbf{The case of $q\in (p,P+1]$.} In this case, we might have $\Theta_{P+1}=1$ for the batch $B_q$.  
As such, we can use \Cref{lem:batch-storing-limit} to obtain that
\begin{align*}
	& \Pr\paren{\Emem{P+1} \mid \Ebatchp{p-1}, \Ememp{p-1}, \Theta_{<p+1}=0} \\
	& = \Exp_{\substack{\sPi^{1:p} ,\sM^{p}}} \bracket{\Pr\paren{\Emem{P+1} \mid \sPi^{1:p}=\Pi^{1:p}, \sM^{p}=M^{p}, \Ebatchp{p-1}, \Ememp{p-1}, \Theta_{<p+1}=0} } \\
	& \leq \frac{1}{2P^2} \tag{using \Cref{lem:batch-storing-limit}}.
\end{align*}

In particular, the last line uses \Cref{lem:batch-storing-limit} by the conditions of $a).$ conditions \Cref{line:induc-prev-prob,line:small-exp-sample} of \Cref{lem:nfl-no-learn} and $b).$ the underlying instance satisfied $\Theta_{\leq p}=0$. These conditions exactly satisfy the requirements of \Cref{lem:batch-storing-limit}.

\noindent
\textbf{The case of $q \leq p$.} 
Note that we have conditioned on the event that $\Theta_{< p+1}=0$. Therefore, we always have
\begin{align*}
	\Pr\paren{\Emem{q} \mid \Ebatchp{p-1}, \Ememp{p-1}, \Theta_{<p+1}=0} = 0.
\end{align*}

\noindent
\textbf{Summarizing the case analysis for \emph{memory obliviousness}.}
By the above cases analysis, we can obtain that 
\begin{align*}
	& \Pr\paren{\neg\, \Ememp{p} \mid \Ebatchp{p-1}, \Ememp{p-1}, \Theta_{<p+1}=0} \\
	& \leq \underbrace{\frac{1}{2P^2}}_{q=P+1} + \underbrace{(P-p+1)\cdot \frac{1}{2P^2}}_{q\in (p,P]} \leq \frac{1}{P}. 
\end{align*}
Therefore, using the conditions \Cref{line:induc-prev-prob} and \Cref{line:small-exp-sample} of \Cref{lem:nfl-no-learn} and conditioning on $\Theta_{<p+1}=0$, the probability for $\ALG$ to be memory-oblivious by the end of the $p$-th pass is at least $(1-\frac{1}{P})\geq (1-\frac{1}{2P})^{3}$ (holds for every $P\geq 2$), i.e.,
\begin{equation}
\label{equ:mem-oblivious-keep}
\Pr\paren{\Ememp{p} \mid \Ememp{p-1}, \Ebatchp{p-1}, \Theta_{\leq p}=0} \geq \paren{1-\frac{1}{2P}}^{3}.
\end{equation} 

\paragraph{Batch oblivious proof.}
We now proceed to the proof of the batch-oblivious property, which completes the building blocks for the proof of \Cref{lem:nfl-no-learn}. We first note that by our analysis for the memory obliviousness, we have

\begin{align*}
	& \Pr\paren{\Ememp{p}, \Ebatchp{p-1}, \Theta_{\leq p}=0} \\
	&\geq \Pr\paren{\Ememp{p}, \Ememp{p-1}, \Ebatchp{p-1}, \Theta_{\leq p}=0}\\
	&\geq \Pr\paren{\Ememp{p} \mid \Ememp{p-1}, \Ebatchp{p-1}, \Theta_{\leq p}=0}\cdot \Pr\paren{\Ememp{p-1}, \Ebatchp{p-1}, \Theta_{\leq p}=0}\\
	&\geq \Pr\paren{\Ememp{p} \mid \Ememp{p-1}, \Ebatchp{p-1}, \Theta_{\leq p}=0}\cdot {\Pr\paren{\Theta_{\leq p}=0\mid \Ebatchp{p-1}, \Ememp{p-1}, \Theta_{<p} = 0}}\\
	& \hspace{220pt} \cdot \Pr\paren{\Ebatchp{p-1}, \Ememp{p-1}, \Theta_{<p} = 0}\\
	&\geq \paren{1-\frac{1}{2P}}^{3} \cdot {\Pr\paren{\Theta_{p}=0\mid \Ebatchp{p-1}, \Ememp{p-1}, \Theta_{<p} = 0}} \cdot \Pr\paren{\Ebatchp{p-1}, \Ememp{p-1}, \Theta_{<p} = 0} \tag{by \Cref{equ:mem-oblivious-keep} and the definition of $\Theta_{\leq p}$}\\
	&\geq \paren{1-\frac{1}{2P}}^{3} \cdot \paren{1-\frac{3}{4P}}\cdot \Pr\paren{\Ebatchp{p-1}, \Ememp{p-1}, \Theta_{<p} = 0} \tag{by batch obliviousness}\\
	&\geq \paren{1-\frac{1}{2P}}^{3} \cdot \paren{1-\frac{3}{4P}} \cdot \paren{1-\frac{1}{2P}}^{10(p-1)}\\
	&\geq \paren{1-\frac{1}{2P}}^{10(p-1)+5}. \tag{using $P\geq 2$}
\end{align*}
As such, the condition for \Cref{lem:batch-learning-limit} is satisfied. Define $\Einform{p}{q}$ as the \emph{event} for the streaming algorithm to produce a pair of memory and transcript $\sPi^{1:p}=\Pi^{1:p}, \sM^{p}=M^{p}$ such that $\Pr(\Theta_{q}=1\mid \sPi^{1:p}=\Pi^{1:p}, \sM^{p}=M^{p}, \Theta_{< p+1}=0)>\frac{1}{2P} + \frac{1}{4P^2}$ by the end of pass $p$ for batch $q$ (in the same way of \Cref{lem:batch-learning-limit}). By applying \Cref{lem:batch-learning-limit}, we have
\begin{align*}
	\Pr\paren{\Einform{p}{q} \mid \Ememp{p}, \Ebatchp{p-1}, \Theta_{\leq p}=0} \leq \frac{1}{2P^2}. 
\end{align*}
Therefore, by a union bound, we have that 
\begin{align*}
\Pr\paren{\Ebatchp{p}\mid \Ememp{p}, \Ebatchp{p-1}, \Theta_{\leq p}=0} \geq \paren{1-\frac{1}{2P}}. 
\end{align*}

\paragraph{Finalizing the proof of \Cref{lem:nfl-no-learn}.} Recall that in the begining of the proof, we have shown that
\begin{align*}
	& \Pr\paren{\Ebatchp{p}, \Ememp{p}, \Theta_{<p+1}=0}\\
	&\geq (1-\frac{3}{4P}) \cdot \paren{1-\frac{1}{2P}}^{10(p-1)}\cdot \Pr\paren{\Ebatchp{p}, \Ememp{p}\mid \Ebatchp{p-1}, \Ememp{p-1}, \Theta_{<p+1} = 0}\\
	&= (1-\frac{3}{4P}) \cdot \paren{1-\frac{1}{2P}}^{10(p-1)}\cdot \Pr\paren{\Ebatchp{p}\mid \Ememp{p}, \Ebatchp{p-1}, \Ememp{p-1}, \Theta_{<p+1}=0}\\
	& \hspace{200pt} \cdot \Pr\paren{\Ememp{p}\mid \Ebatchp{p-1}, \Ememp{p-1}, \Theta_{<p+1}=0}.
\end{align*}
Since $\Ememp{p}$ implies $\Ememp{p-1}$, we can bound the above chain of inequalities as
\begin{align*}
	& \Pr\paren{\Ebatchp{p}, \Ememp{p}, \Theta_{<p+1}=0}\\
	& \geq (1-\frac{3}{4P}) \cdot \paren{1-\frac{1}{2P}}^{10(p-1)} \cdot\paren{1-\frac{1}{2P}}^3 \cdot \Pr\paren{\Ebatchp{p}\mid \Ememp{p}, \Ebatchp{p-1}, \Theta_{<p+1}=0} \tag{the analysis for memory obliviousness} \\
	& \geq (1-\frac{3}{4P}) \cdot \paren{1-\frac{1}{2P}}^{10(p-1)} \cdot\paren{1-\frac{1}{2P}}^3 \cdot \paren{1-\frac{1}{2P}} \tag{analysis for batch obliviousness}\\
	& \geq \paren{1-\frac{1}{2P}}^{10p},
\end{align*}
as desired in the lemma statement. \Qed{lem:nfl-no-learn}
\end{proof}



\subsection{The Radical Case} 
\label{subsec:radical-case}
We now focus on the radical case when the algorithm makes ``too many'' arm pulls for the ``early batches'' in the first $p$ passes while being \emph{oblivious} to these batches.

\begin{lemma}[Radical case]
\label{lem:nfl-sample}
For any integer $p \in [P]$, suppose a streaming algorithm $\ALG$ is memory- and batch-oblivious at the end of the pass $p-1$, and that the underlying instance satisfies $\Theta_{<p} = 0$. 
Additionally, suppose
\[
	\Exp\bracket{\sampleb{p} \mid \Ebatchp{p-1}, \Ememp{p-1}, \Theta_{<p}=0}  > \frac{1}{10^9}\cdot \frac{n}{\gamma_{p+1} \cdot P^{30}};
\]
then, 
\[
	\Exp\bracket{\sample{\ALG} \mid \Ebatchp{p-1}, \Ememp{p-1}, \eventF(p)} > 20000 \cdot C \cdot \frac{n}{\eta_{p}^2}. 
\]
\end{lemma}

Note that the statement of \Cref{lem:nfl-sample} does \emph{not} use the exact ``symmetric'' condition of \Cref{lem:nfl-no-learn} -- we write in this way on purpose, and its usage will be clear in \Cref{subsec:proof-main-theorem}.

We prove \Cref{lem:nfl-sample} for the rest of \Cref{subsec:radical-case}. For technical reasons, for the proof of  \Cref{lem:nfl-sample}, we assume w.log. that on the $p$-th pass, the arm pulls of $\sampleb{p}$ are conducted \emph{before} $\ja$ enters batch $p$, i.e., before the arrival of the arms in the batch $p$. For any algorithm that does \emph{not} satisfy this property, we can re-arrange the order of arm pulls on the $p$-th pass without changing the total number of arm pulls. 

We first show a technical claim that conditioning on any (fixed) transcript and the memory by the end of the $(p-1)$-th pass, the knowledge of the algorithm on $\Theta_{q}$ is \emph{independent} of the arm pulls we used in $\sampleb{q}$.
Conceptually, the claim asserts the simple fact that the arm pulls induced on arms outside $q$, conditioning on the transcript of \emph{all} previous passes, has nothing to do with the algorithm's knowledge for $\Theta_{q}$. We also explicitly use the conditions of $\Ebatchp{p-1}$ and $\Ememp{p-1}$ for technical reason that will be clear later.
\begin{claim}
\label{clm:batch-q-sample-irre}
For any integer $p \in [P]$, let $(\Pi^{1:p-1}, M^{p-1})$ be \emph{any} pair of transcript and memory. Then, for any realization of $\sampleb{q}=s$, i.e., the samples on batches arriving \emph{before} $q$, there is
\begin{align*}
	& \Pr\paren{\Theta_{q}=1 \mid \sampleb{q}=s, \sPi^{1:p-1}=\Pi^{1:p-1}, \sM^{p-1}=M^{p-1}, \Theta_{<p}=0,\Ebatchp{p-1}, \Ememp{p-1}} \\
	& = \Pr\paren{\Theta_{q}=1 \mid \sPi^{1:p-1}=\Pi^{1:p-1}, \sM^{p-1}=M^{p-1}, \Theta_{<p}=0, \Ebatchp{p-1}, \Ememp{p-1}}.
\end{align*}
\end{claim}
\begin{proof}
The proof is an application of the data processing inequality (\Cref{prop:DPI}) with a similar flavor of the rectangle property in communication protocols. Concretely, we want to prove that
\begin{equation}
	\label{equ:theta-before-theta-indep}
	\II\paren{\Theta_{q}; \sPicap{P+1:q+1}^{p} \mid \sPi^{1:p-1}, \sM^{p-1}, \Theta_{<p}=0, \Ebatchp{p-1}, \Ememp{p-1}} = 0.
\end{equation}
Furthermore, observe that $\sampleb{q}$ is a determistic function of $\sPicap{P+1:q+1}^{p}$ \emph{conditioning on} $(\sPi^{1:p-1}, \sM^{p-1}, \Theta_{<p}=0, \Ebatchp{p-1}, \Ememp{p-1})$. Therefore, we have
\begin{align*}
	& \II\paren{\Theta_{q}; \sampleb{q} \mid \sPi^{1:p-1}, \sM^{p-1}, \Theta_{<p}=0, \Ebatchp{p-1}, \Ememp{p-1}} \\
	& \leq \II\paren{\Theta_{q}; \sPicap{P+1:q+1}^{p} \mid \sPi^{1:p-1}, \sM^{p-1}, \Theta_{<p}=0, \Ebatchp{p-1}, \Ememp{p-1}} = 0,
\end{align*}
where the first inequality follows from the data-processing inequality (\Cref{prop:DPI}). Furthermore, since mutual information is non-negative, the above implies that for any realization $(\Pi^{1:p-1}, M^{p-1})$, there is
\begin{align*}
	& \II\paren{\Theta_{q}; \sampleb{q} \mid \sPi^{1:p-1}=\Pi^{1:p-1}, \sM^{p-1}=M^{p-1}, \Theta_{<p}=0, \Ebatchp{p-1}, \Ememp{p-1}} = 0.
\end{align*}
 Therefore, by \Cref{prop:mi-prob-indep}, we have that
\begin{align*}
	& \Pr\paren{\Theta_{q}=1 \mid \sampleb{q}=s, \sPi^{1:p-1}=\Pi^{1:p-1}, \sM^{p-1}=M^{p-1}, \Theta_{<p}=0, \Ebatchp{p-1}, \Ememp{p-1}}\\
	& = \Pr\paren{\Theta_{q}=1 \mid  \sPi^{1:p-1}=\Pi^{1:p-1}, \sM^{p-1}=M^{p-1}, \Theta_{<p}=0, \Ebatchp{p-1}, \Ememp{p-1}}, \tag{by the zero mutual information and \Cref{prop:mi-prob-indep}} 
\end{align*}
which will reach our desired conclusion.

For the rest of this proof, we aim to establish \Cref{equ:theta-before-theta-indep}. To this end, we introduce the random variable for the memory inside the $p$-th pass, we use $\sM^{p}_{>q}$ (resp. $M^{p}_{>q}$) to denote the random variable (resp. the realization) of the memory state when the last time $\ja$ is smaller than all the indices in the $q$-th batch, i.e. the process for the memory state changing in the $p$-th pass can be denoted as $\sM^{p}_{>P}\rightarrow \sM^{p}_{> P-1}\rightarrow \ldots \rightarrow \sM^{p}_{> 1} \rightarrow \sM^{p}_{> 0}=\sM^{p}$.
We further define $\sB{q}$ as the random variable for the \emph{arms} in batch $q$, and $\sB{>q}$ as the random variable for the \emph{arms} in batches $(q, P+1]$. 

For the ease of analysis, we use the following simple trick to \emph{order} the arm pulls induced by $\sPicap{P+1:q+1}^{p}$. In particular, we create a ``imaginary'' process that conducts all samples on batch $q$ \emph{before} $\ja$ visits the batch $(q-1)$ (the next batch in the order of the stream). Note that the ordering does \emph{not} change the value of $\II(\Theta_{q}; \sPicap{P+1:q+1}^{p} \mid \sPi^{1:p-1}, \sM^{p-1}, \Theta_{<p}=0, \Ebatchp{p-1}, \Ememp{p-1})$ since the transcript is permutation-invariate and the arm pulls on $\sPicap{P+1:q+1}^{p}$ are conducted before $\ja$ reaches $B_q$.

We start from batch $P+1$ to ``inductively'' prove the conditional independence between $\Theta_{q}$ and $\sPicap{P+1:r}^{p}$ for $r\in (q, P+1]$. 
Specifically, we first
use chain rule of mutual information to upper-bound the left-hand side of \Cref{equ:theta-before-theta-indep} as follows.
\begin{align*}
& \II\paren{\Theta_{q}; \sPicap{P+1:q+1}^{p} \mid \sPi^{1:p-1}, \sM^{p-1}, \Theta_{<p}=0, \Ebatchp{p-1}, \Ememp{p-1}}\\
& = \II\paren{\Theta_{q}; \sPicap{P+1}^{p}, \sPicap{P:q+1}^{p}\mid \sPi^{1:p-1}, \sM^{p-1}, \Theta_{<p}=0, \Ebatchp{p-1}, \Ememp{p-1}}\\
& = \II\paren{\Theta_{q}; \sPicap{P+1}^{p}\mid \sPi^{1:p-1}, \sM^{p-1}, \Theta_{<p}=0, \Ebatchp{p-1}, \Ememp{p-1}} \\
&\qquad + \II\paren{\Theta_{q}; \sPicap{P:q+1}^{p} \mid \sPi^{1:p-1}, \sM^{p-1}, \sPicap{P+1}^{p}, \Theta_{<p}=0, \Ebatchp{p-1}, \Ememp{p-1}}. \tag{by chain rule}
\end{align*}
We now need to further ``peel off'' random variables from
$\II(\Theta_{q}; \sPicap{P:q+1}^{p} \mid \sPi^{1:p-1}, \sM^{p-1}, \sPicap{P+1}^{p}, \Theta_{<p}=0, \Ebatchp{p-1}, \Ememp{p-1})$ and move the conditions ``forward'' in terms of the batches. To this end, consider the random variable $\sM^{p}_{>P}$, and we observe that
\[\sM^{p}_{>P} \perp \Theta_{q}\mid \sPi^{1:p-1}, \sM^{p-1}, \sPicap{P+1}^{p}, \Theta_{<p}=0,\Ebatchp{p-1}, \Ememp{p-1}\]
since the memory state is uniquely determined \emph{after} the transcript of $\sPicap{P+1}^{p}$ is introduced to $\sPi^{1:p-1}, \sM^{p-1}$ (and since we use the trick to order the transcripts).

Therefore, we can further write the term $\II(\Theta_{q}; \sPicap{P:q+1}^{p} \mid \sPi^{1:p-1}, \sM^{p-1}, \sPicap{P+1}^{p}, \Theta_{<p}=0, \Ebatchp{p-1}, \Ememp{p-1})$ as follows. 
\begin{align*}
	& \II\paren{\Theta_{q}; \sPicap{P:q+1}^{p} \mid \sPi^{1:p-1}, \sM^{p-1}, \sPicap{P+1}^{p}, \Theta_{<p}=0, \Ebatchp{p-1}, \Ememp{p-1}} \\
	&\leq \II\paren{\Theta_{q}; \sPicap{P:q+1}^{p} \mid \sPi^{1:p-1}, \sPicap{P+1}^{p}, \sM^{p-1}, \sM^{p}_{>P}, \Theta_{<p}=0, \Ebatchp{p-1}, \Ememp{p-1}} \tag{condition on independent random variable does not decrease MI}\\
	& = \II\paren{\Theta_{q}; \sPicap{P}^{p}, \sPicap{P-1:q+1}^{p} \mid \sPi^{1:p-1}, \sPicap{P+1}^{p}, \sM^{p-1}, \sM^{p}_{>P}, \Theta_{<p}=0, \Ebatchp{p-1}, \Ememp{p-1}}\\
	& = \II\paren{\Theta_{q}; \sPicap{P}^{p}, \sPicap{P-1:q+1}^{p} \mid \sPi^{1:p-1}, \sPicap{P+1}^{p}, \sM^{p}_{>P}, \Theta_{<p}=0, \Ebatchp{p-1}, \Ememp{p-1}} \\
	&\leq \II\paren{\Theta_{q}; \sPicap{P}^{p}\mid \sPi^{1:p-1}, \sM^{p}_{>P}, \sPicap{P+1}^{p}, \Theta_{<p}=0,\Ebatchp{p-1}, \Ememp{p-1}} \\
	& \qquad + \II\paren{\Theta_{q}; \sPicap{P-1:q+1}^{p} \mid \sPi^{1:p-1}, \sM^{p}_{>P}, \sPicap{P+1:P}^{p}, \Theta_{<p}=0, \Ebatchp{p-1}, \Ememp{p-1}}. \tag{by chain rule}
\end{align*}
Therefore, we can keep performing the above steps, and obtain that:
\begin{align*}
	& \II\paren{\Theta_{q}; \sPicap{P+1:q+1}^{p} \mid \sPi^{1:p-1}, \sM^{p-1}, \Theta_{<p}=0, \Ebatchp{p-1}, \Ememp{p-1}}\\
	& = \sum_{r=P+1}^{q+1} \II\paren{\Theta_{q}; \sPicap{r}^{p}\mid \sPi^{1:p-1}, \sM^{p}_{>r}, \sPicap{P+1:r+1}^{p}, \Theta_{<p}=0, \Ebatchp{p-1}, \Ememp{p-1}}\\
	& \leq \sum_{r=P+1}^{q+1} \II\paren{B_{q}; \sPicap{r}^{p}\mid \sPi^{1:p-1}, \sM^{p}_{>r}, \sPicap{P+1:r+1}^{p}, \Theta_{<p}=0, \Ebatchp{p-1}, \Ememp{p-1}},
\end{align*}
where the last inequality comes from the fact that $\Theta_{q}$ is a deterministic function of $B_{q}$ and by using \Cref{prop:DPI}.

Observe that at each step, in the process with our ordering, the random variable $\sPicap{r}^{p}$ is \emph{determined} by the conditions $\sPi^{1:p-1}, \sM^{p}_{>r}, \sPicap{P+1:r+1}^{p}$. As such, we have that
\[B_{q} \perp \sPicap{r}^{p}\mid \sPi^{1:p-1}, \sM^{p}_{>r}, \sPicap{P+1:r+1}^{p}, \Theta_{<p}=0, \Ebatchp{p-1}, \Ememp{p-1},\]
which implies
\begin{align*}
	& \II\paren{\Theta_{q}; \sPicap{P+1:q+1}^{p} \mid \sPi^{1:p-1}, \sM^{p-1}, \Theta_{<p}=0, \Ebatchp{p-1}, \Ememp{p-1}}\\
	& \leq \sum_{r=P+1}^{q+1} \II\paren{B_{q}; \sPicap{r}^{p}\mid \sPi^{1:p-1}, \sM^{p}_{>r}, \sPicap{P+1:r+1}^{p}, \Theta_{<p}=0,\Ebatchp{p-1}, \Ememp{p-1}} = 0,
\end{align*}
which is as desired by \Cref{equ:theta-before-theta-indep}.
\myqed{\Cref{clm:batch-q-sample-irre}}
\end{proof}

We now proceed to the main technical lemma to prove \Cref{lem:nfl-sample}: we show that conditioning on any transcript and memory that satisfies the assumptions of \Cref{lem:nfl-sample}, the expected number of samples for $\sampleb{p}$ has to be high even if we add the \emph{extra condition} of $\Theta_{p}=1$. Note that a standard total expectation calculation only leads to the reverse direction, and the correctness of our case crucially relies on the lower bound for batch-obliviousness.
\begin{lemma}
\label{lem:high-sample-extra-cond}
Let $(\Pi^{1:p-1}, M^{p-1})$ be a pair of transcript and memory of a streaming algorithm after $(p-1)$ passes, there is
\begin{align*}
	& \Exp\bracket{\sampleb{p} \mid \sPi^{1:p-1}=\Pi^{1:p-1}, \sM^{p-1}=M^{p-1}, \Theta_{<p}=0, \Theta_{p}=1, \Ebatchp{p-1}, \Ememp{p-1}} \\
	& \geq \frac{1}{2}\cdot \Exp\bracket{\sampleb{p} \mid \sPi^{1:p-1}=\Pi^{1:p-1}, \sM^{p-1}=M^{p-1}, \Theta_{<p}=0, \Ebatchp{p-1}, \Ememp{p-1}}.
\end{align*}
\end{lemma}
\begin{proof}
To avoid clutter, for the given pair of transcript and memory $(\Pi^{1:p-1}, M^{p-1})$ that satisfies the lemma statement, we define random variable
\[\cS := \sampleb{p} \mid \sPi^{1:p-1}=\Pi^{1:p-1}, \sM^{p-1}=M^{p-1}, \Theta_{<p}=0, \Ebatchp{p-1}, \Ememp{p-1}\]
as a notation that is self-contain in this proof. In this way, by picking realizations for $\cS=s$, we mean $\paren{\sampleb{p}=s \mid \sPi^{1:p-1}=\Pi^{1:p-1}, \sM^{p-1}=M^{p-1}, \Theta_{<p}=0, \Ebatchp{p-1}, \Ememp{p-1}}$. By Bayes' rule, for any realization of $\cS=s$, we have
\begin{align*}
	\Pr\paren{\cS=s\mid \Theta_{p}=1} = \frac{\Pr\paren{\Theta_{p}=1\mid \cS=s}\cdot \Pr\paren{\cS=s}}{\Pr\paren{\Theta_{p}=1}}.
\end{align*}
As such, by using \Cref{clm:batch-q-sample-irre} with $q=p$, we have that 
\begin{align*}
	& \Pr\paren{\Theta_{p}=1\mid \cS=s} \\
	&= \Pr\paren{\Theta_{p}=1\mid \sampleb{p}=s, \sPi^{1:p-1}=\Pi^{1:p-1}, \sM^{p-1}=M^{p-1}, \Theta_{<p}=0, \Ebatchp{p-1}, \Ememp{p-1}}\\
	& = \Pr\paren{\Theta_{p}=1\mid \sPi^{1:p-1}=\Pi^{1:p-1}, \sM^{p-1}=M^{p-1}, \Theta_{<p}=0, \Ebatchp{p-1}, \Ememp{p-1}}.
\end{align*}
Therefore, for any choice of $s$, we have the bound 
\begin{align*}
	& \Pr\paren{\cS=s \mid \Theta_{p}=1} \\
	&= \frac{\Pr\paren{\Theta_{p}=1\mid \sPi^{1:p-1}=\Pi^{1:p-1}, \sM^{p-1}=M^{p-1}, \Theta_{<p}=0, \Ebatchp{p-1}, \Ememp{p-1}}\cdot \Pr\paren{\cS=s}}{\Pr\paren{\Theta_{p}=1}}\\
	& \geq \frac{\paren{\frac{1}{2P}-\frac{1}{4P^2}}}{\Pr\paren{\Theta_{q}=1}}\cdot \Pr\paren{\cS=s} \tag{by the assumption of batch-obliviousness}\\
	& \geq \frac{1}{2}\cdot \Pr\paren{\cS=s}.
\end{align*}
Therefore, we have 
\begin{align*}
	& \Exp\bracket{\sampleb{p} \mid \sPi^{1:p-1}=\Pi^{1:p-1}, \sM^{p-1}=M^{p-1}, \Theta_{<p}=0, \Theta_{p}=1, \Ebatchp{p-1}, \Ememp{p-1}} \\
	& = \Exp\bracket{\cS \mid \Theta_{p}=1} \tag{change of notation}\\
	& = \sum_{s} s\cdot \Pr\paren{\cS=s \mid \Theta_{p}=1} \\
	& \geq \sum_{s} s\cdot \frac{1}{2}\cdot \Pr\paren{\cS=s} \\
	& = \frac{1}{2}\cdot \Exp\bracket{\sampleb{p} \mid \sPi^{1:p-1}=\Pi^{1:p-1}, \sM^{p-1}=M^{p-1}, \Theta_{<p}=0,  \Ebatchp{p-1}, \Ememp{p-1}},
\end{align*}
as desired.
\end{proof}

\begin{proof}[\textbf{Proof of \Cref{lem:nfl-sample}}]
We first lower bound the total expected number of arm pulls with the expected number of arm pulls \emph{restricting to} the first $P-p+2$ batches, and write the expectation as the average case of the choices of $(\Pi^{1:p-1}, M^{p-1})$. 
\begin{align*}
	& \Exp \bracket{\sample{\ALG} \mid \eventF(p), \Ebatchp{p-1},\Ememp{p-1}}\\
	& = \Exp \Exp_{\sPi^{1:p-1}, \sM^{p-1}}\bracket{\sample{\ALG} \mid \Pi^{1:p-1}, M^{p-1}, \eventF(p), \Ebatchp{p-1},\Ememp{p-1}}\\
	& \geq \Exp \Exp_{\sPi^{1:p-1}, \sM^{p-1}}\bracket{\sampleb{p} \mid \Pi^{1:p-1}, M^{p-1}, \eventF(p), \Ebatchp{p-1},\Ememp{p-1}},
\end{align*}
where 
the last inequality is due to $\sampleb{p}$ always counts a subset of arm pulls of $\sample{\ALG}$ for any fixed transcript and memory $(\Pi^{1:p-1}, M^{p-1})$. 
By applying \Cref{lem:high-sample-extra-cond} to \emph{every} choice of batch- and memory-oblivious $(\Pi^{1:p-1}, M^{p-1})$, we have
\begin{align*}
	& \Exp \Exp_{\sPi^{1:p-1}, \sM^{p-1}} \bracket{\sampleb{p} \mid \Pi^{1:p-1}, M^{p-1}, \Theta_{<p}=0, \Theta_{p}=1, \Ebatchp{p-1},\Ememp{p-1}}\\
	& \geq \frac{1}{2}\cdot \Exp \Exp_{\sPi^{1:p-1}, \sM^{p-1}}\bracket{\sampleb{p} \mid \Pi^{1:p-1}, M^{p-1}, \Theta_{<p}=0, \Ebatchp{p-1},\Ememp{p-1}}.
\end{align*}
To avoid clutter, we can combine the above inequalities and re-write them in the form of 
\begin{align*}
	& \Exp \bracket{\sample{\ALG} \mid \eventF(p), \Ebatchp{p-1},\Ememp{p-1}} \\
	& \geq \frac{1}{2}\cdot \Exp\bracket{\sampleb{p} \mid \Theta_{<p}=0, \Ebatchp{p-1},\Ememp{p-1}} .
\end{align*}
Therefore, by the assumption of \Cref{lem:nfl-sample}, we have
\begin{align*}
	& \Exp \bracket{\sample{\ALG} \mid \eventF(p), \Ebatchp{p-1},\Ememp{p-1}} \\
	& \geq \Exp \bracket{\sampleb{p} \mid \Theta_{<p}=0, \Theta_{p}=1, \Ebatchp{p-1},\Ememp{p-1}} \\
	& \geq \frac{1}{2}\cdot \bracket{\sampleb{p} \mid \Theta_{<p}=0, \Ebatchp{p-1},\Ememp{p-1}} \tag{by \Cref{lem:high-sample-extra-cond}}\\
	& > \frac{1}{10^{10}}\cdot \frac{n}{\gamma_{p+1}\cdot P^{30}} \tag{by the condition of \Cref{lem:nfl-sample}}\\
	& = \frac{1}{10^{10}}\cdot \frac{n}{\eta^2_{p}\cdot P^{30}} \cdot (6C\cdot P)^{30} \tag{by the construction $\eta_{p+1}=\paren{\frac{1}{6CP}}^{15}\cdot \eta_{p}$}\\
	& > 20000C \cdot \frac{n}{\eta^2_{p}} \tag{$6^{30}/10^{10}>20000$ with $C\geq 1$},
\end{align*}
as desired by \Cref{lem:nfl-sample}. \Qed{lem:nfl-sample} 
\end{proof}

\subsection{Putting Everything Together: Proof of~\Cref{thm:main}}
\label{subsec:proof-main-theorem}
We now prove \Cref{thm:main} with Lemmas~\ref{lem:nfl-no-learn} and \ref{lem:nfl-sample}. We remind the readers that we use $\ALG$ to denote the streaming algorithm. Note that in the beginning of the first pass, $\ALG$ is necessarily memory- and batch-oblivious since there is $\Pi^{0}=\emptyset$ and $M^{0}=\emptyset$. Therefore, by Lemma~\ref{lem:nfl-sample}, if the algorithm enters the \emph{radical case}, there is
\begin{align*}
	\expect{\sample{\ALG}\mid \eventF(1)} = \expect{\sample{\ALG}\mid \Theta_{1}=1} > C \cdot \frac{n}{\eta_{1}^2},
\end{align*}
which breaks the sample complexity requirement in \Cref{thm:main}. Therefore, $\ALG$ must use the \emph{conservative case} for the first pass. 

Starting from the second pass, 
we argue that no pass should use the \emph{radical case} if $\ALG$ is to follow the upper bound on the sample complexity as required by \Cref{thm:main}. Suppose $\ptilde$ is the first pass that the algorithm enters the \emph{radical case}, and since we have the base case of $p=1$ and the condition of \Cref{lem:nfl-no-learn} (conservative case) being satisfied before pass $\ptilde$, there is
\begin{align*}
	\Pr\paren{\Ebatchp{\ptilde-1}, \Ememp{\ptilde-1}, \Theta_{<\ptilde}=0} \geq \paren{1-\frac{1}{2P}}^{10(\ptilde-1)}. 
\end{align*}
We use the above result to lower bound the probability for $\Pr(\Ebatchp{\ptilde-1}, \Ememp{\ptilde-1} \mid \eventF(\ptilde))$, which will eventually lead to a lower bound on $\expect{\sample{\ALG}\mid \eventF(\ptilde)}$ that breaks the limit of samples. 

To this end, we first show the following technical claim that allows us to ``drop'' conditions on $\Theta_{\ptilde}$ conditioning on $\Ebatchp{\ptilde-1}$ and $\Ememp{\ptilde-1}$. Intuitively, such a claim is true by the obliviousness of the transcipt on $\Theta_{\ptilde}$, which is similar-in-spirit with \Cref{lem:high-sample-extra-cond}.
\begin{claim}
\label{clm:oblivious-drop-theta}
The following statement is true:
\begin{align*}
	\Pr\paren{\Ebatchp{\ptilde-1}, \Ememp{\ptilde-1} \mid \eventF(\ptilde)} \geq \frac{1}{2}\cdot \Pr\paren{\Ebatchp{\ptilde-1}, \Ememp{\ptilde-1}, \Theta_{<\ptilde}=0}. 
\end{align*}
\end{claim}
\begin{proof}
We first lower bound the probability by expanding the terms as follows.
\begin{align*}
	\Pr\paren{\Ebatchp{\ptilde-1}, \Ememp{\ptilde-1} \mid \eventF(\ptilde)} &= \frac{\Pr\paren{\Ebatchp{\ptilde-1}, \Ememp{\ptilde-1}, \eventF(\ptilde)}}{\Pr\paren{\eventF(\ptilde)}} \\
	&\geq \Pr\paren{\Ebatchp{\ptilde-1}, \Ememp{\ptilde-1}, \eventF(\ptilde)}\cdot 2P \tag{$\Pr\paren{\eventF(\ptilde)}\leq 1/2P$}\\
	& = \Pr\paren{\Ebatchp{\ptilde-1}, \Ememp{\ptilde-1}, \Theta_{\ptilde}=1, \Theta_{<\ptilde}=0}\cdot 2P \\
	& = \Pr\paren{\Ebatchp{\ptilde-1}, \Ememp{\ptilde-1}, \Theta_{<\ptilde}=0 \mid \Theta_{\ptilde}=1}\cdot \Pr\paren{\Theta_{\ptilde}=1} \cdot 2P \\
	&= \Pr\paren{\Ebatchp{\ptilde-1}, \Ememp{\ptilde-1}, \Theta_{<\ptilde}=0 \mid \Theta_{\ptilde}=1} \tag{$\Pr(\Theta_{\ptilde}=1)=\frac{1}{2P}$ by $\ptilde\leq p$}.
\end{align*}
Our goal now is to lower bound the term $\Pr\paren{\Ebatchp{\ptilde-1}, \Ememp{\ptilde-1}, \Theta_{<\ptilde}=0 \mid \Theta_{\ptilde}=1}$. By Bayes' rule and the batch-obliviousness, we have
\begin{align*}
	& \Pr\paren{\Ebatchp{\ptilde-1}, \Ememp{\ptilde-1}, \Theta_{<\ptilde}=0 \mid \Theta_{\ptilde}=1}\\
	& = \frac{\Pr\paren{\Theta_{\ptilde}=1 \mid \Ebatchp{\ptilde-1}, \Ememp{\ptilde-1}, \Theta_{<\ptilde}=0} \cdot \Pr\paren{\Ebatchp{\ptilde-1}, \Ememp{\ptilde-1}, \Theta_{<\ptilde}=0}}{\Pr\paren{\Theta_{\ptilde}=1}} \tag{Bayes' rule}\\
	& \geq \frac{\paren{\frac{1}{4P}} \cdot \Pr\paren{\Ebatchp{\ptilde-1}, \Ememp{\ptilde-1} \mid \Theta_{<\ptilde}=0}}{\frac{1}{2P}} \tag{by the batch-oblivious condition}\\
	& \geq \frac{1}{2}\cdot \Pr\paren{\Ebatchp{\ptilde-1}, \Ememp{\ptilde-1}, \Theta_{<\ptilde}=0},
\end{align*}
as desired.
\end{proof}

We now establish the lower bound on the expected sample for pass $\ptilde$. By \Cref{clm:oblivious-drop-theta}, we have that

\begin{equation}
\label{equ:prob-bad-trans}
\begin{aligned}
\Pr\paren{\Ebatchp{\ptilde-1}, \Ememp{\ptilde-1}\mid \eventF(\ptilde)} &\geq \frac{1}{2}\cdot\paren{1-\frac{1}{2P}}^{10(\ptilde-1)}\\
&\geq \frac{1}{2}\cdot\paren{1-\frac{1}{2P}}^{10(P-1)}\\
& > \frac{1}{1000},
\end{aligned}
\end{equation}
where 
the first inequality uses \Cref{clm:oblivious-drop-theta} and the lower bound on $\Pr(\Ebatchp{\ptilde-1}, \Ememp{\ptilde-1}, \Theta_{<\ptilde}=0)$, and the last inequality is obtained by using $\paren{1-\frac{1}{2P}}^{10P-10}>\frac{1}{500}$ for any $P\geq 2$. Therefore, we can bound the sample complexity of the algorithm if it enters the \emph{radical case} on the $\ptilde$-th pass as follows. 
\begin{align*}
& \expect{\sample{\ALG} \mid \eventF(\ptilde)}\\
& \geq \expect{\sample{\ALG} \mid \eventF(\ptilde), \, \Ebatchp{\ptilde-1}, \, \Ememp{\ptilde-1} } \\
& \qquad \cdot \Pr\paren{ \Ebatchp{\ptilde-1}, \, \Ememp{\ptilde-1} \mid \eventF(\ptilde)}\\
& > 20000 C \cdot \frac{n}{\eta^2_{p}} \cdot \frac{1}{1000} \tag{by \Cref{lem:nfl-sample} and the lower bound of $\Pr\paren{\Ebatchp{\ptilde-1}, \Ememp{\ptilde-1}\mid \eventF(\ptilde)}$}\\
& > C \cdot \frac{n}{\eta_{p}^2},
\end{align*}
which breaks the requirement of sample complexity bound in \Cref{thm:main}. As such, to keep the promise on the sample complexity, $\ALG$ has to be in the \emph{conservative case} for \emph{all} $P$ passes. 

Now, we can apply the calculation in \Cref{equ:prob-bad-trans} again to argue that with probability strictly more than $\frac{1}{1000}$, after the $P$-th pass, we obtain transcript and memory that are memory- and batch-oblivious. As such, no arm with a mean reward strictly more than $1/2$ will be in the memory of $\ALG$, which means the success probability is strictly less than $\frac{999}{1000}$.

\begin{remark}
Observe that our lower bound generalizes to a sample complexity with additional $\polylog{\frac{1}{\eta_{p}}}$ multiplicative factors, i.e., we can prove the lower bound of memory and success probability in \Cref{thm:main} with the condition of 
\[	\Exp\bracket{\sample{\ALG} \mid \eventF(p)} \leq C \cdot \frac{n}{\eta_p^2}\cdot \polylog{\frac{1}{\eta_{p}}}.\]
In particular, if we further increase the gap between $\eta_p$ in different batches, e.g., if we use $\eta_p = \paren{\frac{1}{6 \, C \cdot P}}^{20p}$ in \Cref{eq:gap-parameter}, we can bring an extra multiplicative term of $\poly(P)$ to the sample complexity in our proof of \Cref{lem:nfl-sample}. By our choice of parameters, we have $P=\Omega(\frac{\log{(1/\DeltaT)}}{\log\log{(1/\DeltaT)}})$, where $\DeltaT\leq \eta_{p}$ for any $p\in [P+1]$. This leads to the desired bound on sample complexity. The observation also strengthens our main lower bound result to an expected sample complexity of $O(\frac{n}{\Delta^2}\cdot \polylog{\frac{1}{\Delta}})$.
\end{remark}


\bibliographystyle{alpha}
\bibliography{general}

\begin{thebibliography}{EDMM02}

\bibitem[AAAK17]{AgarwalAAK17}
Arpit Agarwal, Shivani Agarwal, Sepehr Assadi, and Sanjeev Khanna.
\newblock Learning with limited rounds of adaptivity: Coin tossing, multi-armed
  bandits, and ranking from pairwise comparisons.
\newblock In {\em Proceedings of the 30th Conference on Learning Theory, {COLT}
  2017, Amsterdam, The Netherlands, 7-10 July 2017}, pages 39--75, 2017.

\bibitem[ACE09]{AgarwalCE09}
Deepak Agarwal, Bee{-}Chung Chen, and Pradheep Elango.
\newblock Explore/exploit schemes for web content optimization.
\newblock In Wei Wang, Hillol Kargupta, Sanjay Ranka, Philip~S. Yu, and Xindong
  Wu, editors, {\em {ICDM} 2009, The Ninth {IEEE} International Conference on
  Data Mining, Miami, Florida, USA, 6-9 December 2009}, pages 1--10. {IEEE}
  Computer Society, 2009.

\bibitem[ACNS23]{ACNS23}
Anders Aamand, Justin~Y. Chen, Huy~L{\^{e}} Nguyen, and Sandeep Silwal.
\newblock Improved space bounds for learning with experts.
\newblock {\em CoRR}, abs/2303.01453, 2023.

\bibitem[AKP22]{AgarwalKP22}
Arpit Agarwal, Sanjeev Khanna, and Prathamesh Patil.
\newblock A sharp memory-regret trade-off for multi-pass streaming bandits.
\newblock In Po{-}Ling Loh and Maxim Raginsky, editors, {\em Conference on
  Learning Theory, 2-5 July 2022, London, {UK}}, volume 178 of {\em Proceedings
  of Machine Learning Research}, pages 1423--1462. {PMLR}, 2022.

\bibitem[AKR21]{aziz2021multi}
Maryam Aziz, Emilie Kaufmann, and Marie-Karelle Riviere.
\newblock On multi-armed bandit designs for dose-finding clinical trials.
\newblock {\em The Journal of Machine Learning Research}, 22(1):686--723, 2021.

\bibitem[AW20]{AssadiW20}
Sepehr Assadi and Chen Wang.
\newblock Exploration with limited memory: streaming algorithms for coin
  tossing, noisy comparisons, and multi-armed bandits.
\newblock In Konstantin Makarychev, Yury Makarychev, Madhur Tulsiani, Gautam
  Kamath, and Julia Chuzhoy, editors, {\em Proccedings of the 52nd Annual {ACM}
  {SIGACT} Symposium on Theory of Computing, {STOC} 2020, Chicago, IL, USA,
  June 22-26, 2020}, pages 1237--1250. {ACM}, 2020.

\bibitem[AW22]{AWneurips22}
Sepehr Assadi and Chen Wang.
\newblock Single-pass streaming lower bounds for multi-armed bandits
  exploration with instance-sensitive sample complexity.
\newblock In {\em Advances in Neural Information Processing Systems 35: Annual
  Conference on Neural Information Processing Systems 2022, NeurIPS 2022 (to
  appear)}, 2022.

\bibitem[BM07]{BertsimasM07}
Dimitris Bertsimas and Adam~J. Mersereau.
\newblock A learning approach for interactive marketing to a customer segment.
\newblock {\em Oper. Res.}, 55(6):1120--1135, 2007.

\bibitem[CK20]{ChaudhuriK20}
Arghya~Roy Chaudhuri and Shivaram Kalyanakrishnan.
\newblock Regret minimisation in multi-armed bandits using bounded arm memory.
\newblock In {\em The Thirty-Fourth {AAAI} Conference on Artificial
  Intelligence, {AAAI} 2020, The Thirty-Second Innovative Applications of
  Artificial Intelligence Conference, {IAAI} 2020, The Tenth {AAAI} Symposium
  on Educational Advances in Artificial Intelligence, {EAAI} 2020, New York,
  NY, USA, February 7-12, 2020}, pages 10085--10092. {AAAI} Press, 2020.

\bibitem[CLQ17]{ChenLQ17}
Lijie Chen, Jian Li, and Mingda Qiao.
\newblock Nearly instance optimal sample complexity bounds for top-k arm
  selection.
\newblock In Aarti Singh and Xiaojin~(Jerry) Zhu, editors, {\em Proceedings of
  the 20th International Conference on Artificial Intelligence and Statistics,
  {AISTATS} 2017, 20-22 April 2017, Fort Lauderdale, FL, {USA}}, volume~54 of
  {\em Proceedings of Machine Learning Research}, pages 101--110. {PMLR}, 2017.

\bibitem[CT06]{CoverT06}
Thomas~M. Cover and Joy~A. Thomas.
\newblock {\em Elements of information theory {(2.} ed.)}.
\newblock Wiley, 2006.

\bibitem[EDMM02]{Even-Dar+02}
Eyal Even-Dar, Shie Mannor, and Yishay Mansour.
\newblock {PAC Bounds for Multi-Armed Bandit and Markov Decision Processes}.
\newblock In {\em COLT}, 2002.

\bibitem[EMM06]{Even-Dar+06}
Eyal Even{-}Dar, Shie Mannor, and Yishay Mansour.
\newblock Action elimination and stopping conditions for the multi-armed bandit
  and reinforcement learning problems.
\newblock {\em Journal of Machine Learning Research}, 7:1079--1105, 2006.

\bibitem[HYZ24]{HYZ24}
Yuchen He, Zichun Ye, and Chihao Zhang.
\newblock Understanding memory-regret trade-off for streaming stochastic
  multi-armed bandits.
\newblock {\em CoRR}, abs/2405.19752, 2024.

\bibitem[JHTX21]{JinH0X21}
Tianyuan Jin, Keke Huang, Jing Tang, and Xiaokui Xiao.
\newblock Optimal streaming algorithms for multi-armed bandits.
\newblock In Marina Meila and Tong Zhang, editors, {\em Proceedings of the 38th
  International Conference on Machine Learning, {ICML} 2021, 18-24 July 2021,
  Virtual Event}, volume 139 of {\em Proceedings of Machine Learning Research},
  pages 5045--5054. {PMLR}, 2021.

\bibitem[JMNB14]{JamiesonMNB14}
Kevin~G. Jamieson, Matthew Malloy, Robert~D. Nowak, and S{\'{e}}bastien Bubeck.
\newblock lil' {UCB} : An optimal exploration algorithm for multi-armed
  bandits.
\newblock In Maria{-}Florina Balcan, Vitaly Feldman, and Csaba
  Szepesv{\'{a}}ri, editors, {\em Proceedings of The 27th Conference on
  Learning Theory, {COLT} 2014, Barcelona, Spain, June 13-15, 2014}, volume~35
  of {\em {JMLR} Workshop and Conference Proceedings}, pages 423--439.
  JMLR.org, 2014.

\bibitem[KKS13]{KarninKS13}
Zohar~Shay Karnin, Tomer Koren, and Oren Somekh.
\newblock Almost optimal exploration in multi-armed bandits.
\newblock In {\em Proceedings of the 30th International Conference on Machine
  Learning, {ICML} 2013, Atlanta, GA, USA, 16-21 June 2013}, volume~28 of {\em
  {JMLR} Workshop and Conference Proceedings}, pages 1238--1246. JMLR.org,
  2013.

\bibitem[KS10]{KalyanakrishnanSt10}
Shivaram Kalyanakrishnan and Peter Stone.
\newblock {Efficient Selection of Multiple Bandit Arms: Theory and Practice}.
\newblock In {\em ICML}, 2010.

\bibitem[KZ23]{KarpovZhangAAAI23}
Nikolai Karpov and Qin Zhang.
\newblock Communication-efficient collaborative best arm identification.
\newblock In {\em Proc. AAAI Conference on Artificial Intelligence (AAAI 23)},
  2023.

\bibitem[KZZ20]{Karpov0020}
Nikolai Karpov, Qin Zhang, and Yuan Zhou.
\newblock Collaborative top distribution identifications with limited
  interaction (extended abstract).
\newblock In Sandy Irani, editor, {\em 61st {IEEE} Annual Symposium on
  Foundations of Computer Science, {FOCS} 2020, Durham, NC, USA, November
  16-19, 2020}, pages 160--171. {IEEE}, 2020.

\bibitem[LCLS10]{LiCLS10}
Lihong Li, Wei Chu, John Langford, and Robert~E. Schapire.
\newblock A contextual-bandit approach to personalized news article
  recommendation.
\newblock In Michael Rappa, Paul Jones, Juliana Freire, and Soumen Chakrabarti,
  editors, {\em Proceedings of the 19th International Conference on World Wide
  Web, {WWW} 2010, Raleigh, North Carolina, USA, April 26-30, 2010}, pages
  661--670. {ACM}, 2010.

\bibitem[LSPY18]{LiauSPY18}
David Liau, Zhao Song, Eric Price, and Ger Yang.
\newblock Stochastic multi-armed bandits in constant space.
\newblock In Amos~J. Storkey and Fernando P{\'{e}}rez{-}Cruz, editors, {\em
  International Conference on Artificial Intelligence and Statistics, {AISTATS}
  2018, 9-11 April 2018, Playa Blanca, Lanzarote, Canary Islands, Spain},
  volume~84 of {\em Proceedings of Machine Learning Research}, pages 386--394.
  {PMLR}, 2018.

\bibitem[LZWL23]{LZWL23}
Shaoang Li, Lan Zhang, Junhao Wang, and Xiang{-}Yang Li.
\newblock Tight memory-regret lower bounds for streaming bandits.
\newblock {\em CoRR}, abs/2306.07903, 2023.

\bibitem[MPK21]{MaitiPK21}
Arnab Maiti, Vishakha Patil, and Arindam Khan.
\newblock Multi-armed bandits with bounded arm-memory: Near-optimal guarantees
  for best-arm identification and regret minimization.
\newblock In Marc'Aurelio Ranzato, Alina Beygelzimer, Yann~N. Dauphin, Percy
  Liang, and Jennifer~Wortman Vaughan, editors, {\em Advances in Neural
  Information Processing Systems 34: Annual Conference on Neural Information
  Processing Systems 2021, NeurIPS 2021, December 6-14, 2021, virtual}, pages
  19553--19565, 2021.

\bibitem[MT04]{MannorTs04}
Shie Mannor and John~N Tsitsiklis.
\newblock {The Sample Complexity of Exploration in the Multi-Armed Bandit
  Problem}.
\newblock {\em Journal of Machine Learning Research}, 5:623--648, 2004.

\bibitem[PZ23]{PengZ23}
Binghui Peng and Fred Zhang.
\newblock Online prediction in sub-linear space.
\newblock In Nikhil Bansal and Viswanath Nagarajan, editors, {\em Proceedings
  of the 2023 {ACM-SIAM} Symposium on Discrete Algorithms, {SODA} 2023,
  Florence, Italy, January 22-25, 2023}, pages 1611--1634. {SIAM}, 2023.

\bibitem[SBF17]{SchwartzBF17}
Eric~M. Schwartz, Eric~T. Bradlow, and Peter~S. Fader.
\newblock Customer acquisition via display advertising using multi-armed bandit
  experiments.
\newblock {\em Mark. Sci.}, 36(4):500--522, 2017.

\bibitem[SWXZ22]{SrinivasWXZ22}
Vaidehi Srinivas, David~P. Woodruff, Ziyu Xu, and Samson Zhou.
\newblock Memory bounds for the experts problem.
\newblock In Stefano Leonardi and Anupam Gupta, editors, {\em {STOC} '22: 54th
  Annual {ACM} {SIGACT} Symposium on Theory of Computing, Rome, Italy, June 20
  - 24, 2022}, pages 1158--1171. {ACM}, 2022.

\bibitem[TZZ19]{Tao0019}
Chao Tao, Qin Zhang, and Yuan Zhou.
\newblock Collaborative learning with limited interaction: Tight bounds for
  distributed exploration in multi-armed bandits.
\newblock In David Zuckerman, editor, {\em 60th {IEEE} Annual Symposium on
  Foundations of Computer Science, {FOCS} 2019, Baltimore, Maryland, USA,
  November 9-12, 2019}, pages 126--146. {IEEE} Computer Society, 2019.

\bibitem[VBW15]{villar2015multi}
Sof{\'\i}a~S Villar, Jack Bowden, and James Wason.
\newblock Multi-armed bandit models for the optimal design of clinical trials:
  benefits and challenges.
\newblock {\em Statistical science: a review journal of the Institute of
  Mathematical Statistics}, 30(2):199, 2015.

\bibitem[Wan23]{Wang23Regret}
Chen Wang.
\newblock Tight regret bounds for single-pass streaming multi-armed bandits.
\newblock In {\em Proceedings of the 40th International Conference on Machine
  Learning, {ICML} 2023 (To appear)}, Proceedings of Machine Learning Research,
  2023.

\end{thebibliography}

\appendix

\clearpage
\section{Standard Technical Tools}
\label{app:info-theoretic-facts}


\subsection{Statistical Distances}
\label{sub-app:stat-dist}
We introduce the widely-used statistical distance notions of Kullback–Leibler divergence (KL divergence) and total variation distance (TVD) in this section.

\paragraph{KL divergence and its properties.} We start with introducing the Kullback–Leibler divergence (KL divergence) and its properties.

\begin{definition}[KL divergence]
\label{def:kl-div}
Let $X$ and $Y$ be two discrete random variables supported over the same $\Omega$, and let their distributions be $\mu_{X}$ and $\mu_{Y}$.The KL divergence between $X$ and $Y$, denoted as $\kl{X}{Y}$, is defined as 
\begin{align*}
\kl{X}{Y} = \sum_{\omega\in \Omega} \mu_{X}(\omega)\log\paren{\frac{\mu_{X}(\omega)}{\mu_{Y}(\omega)}}.
\end{align*}
\end{definition}


\paragraph{Total variation distance and its properties.} 
Similar to the KL divergence we defined above, the total variation distance (TVD) is another statistical distance between two distributions. 
\begin{definition}
\label{def:tvd}
Let $X$ and $Y$ be two random variables supported over the same $\Omega$, and let $\mu_X$ and $\mu_Y$ be their probability measures. The total variation distance (TVD) is between $X$ and $Y$ is defined as
\begin{align*}
\tvd{X}{Y} = \sup_{\Omega'\subseteq \Omega}\card{\mu_X(\Omega')-\mu_Y(\Omega')}.
\end{align*}
In particular, when the random variables are discrete, we have
\begin{align*}
\tvd{X}{Y} = \frac{1}{2}\sum_{\omega\in \Omega}\card{\mu_X(\omega)-\mu_Y(\omega)}.
\end{align*}
\end{definition}

\subsection{Statistical Distances and Their Properties}
\label{sub-app:info-theoretic-facts}

We shall use the following standard properties of KL-divergence and TVD defined in~\Cref{sub-app:stat-dist}. For the proof of this results, see the excellent textbook by Cover and Thomas~\cite{CoverT06}. 

The following facts state the chain rule property and convexity of KL-divergence. 

\begin{fact}[Chain rule of KL divergence]
\label{fact:kl-chain-rule}
For any random variables $X=(X_{1}, X_{2})$ and $Y=(Y_{1},Y_{2})$ be two random variables, 
\begin{align*}
\kl{X}{Y} = \kl{X_{1}}{Y_{1}} + \Exp_{x \sim X_1} \kl{X_{2}\mid X_{1}=x}{Y_{2}\mid Y_{1}=x}.
\end{align*}
\end{fact}

\begin{fact}[Convexity KL-divergence]
\label{fact:kl-convexity}
For any distributions $\mu_1,\mu_2$ and $\nu_1,\nu_2$ and any $\lambda \in (0,1)$, 
\begin{align*}
\kl{\lambda \cdot \mu_1 + (1-\lambda) \cdot \mu_2}{\lambda \cdot \nu_1 + (1-\lambda) \cdot \nu_2} \leq \lambda \cdot \kl{\mu_1}{\nu_1} + (1-\lambda) \cdot \kl{\mu_2}{\nu_2}. 
\end{align*}
\end{fact}

\begin{fact}[Conditioning cannot decrease KL-divergence]\label{fact:kl-conditioning}
	For any random variables $X,Y,Z$, 
	\[
		\kl{X}{Y} \leq \Exp_{z \sim Z} \kl{X \mid Z=z}{Y \mid Z=z}. 
	\]
\end{fact}

Pinsker's inequality relates KL-divergence to TVD. 

\begin{fact}[Pinsker's inequality]
\label{fact:pinsker}
For any random variables $X$ and $Y$ supported over the same $\Omega$, 
\begin{align*}
\tvd{X}{Y} \leq \sqrt{\frac{1}{2}\cdot \kl{X}{Y}}.
\end{align*}
\end{fact} 

The following fact characterizes the error of MLE for the source of a sample based on the TVD of the originating distributions. 

\begin{fact}
\label{fact:distinguish-tvd}
	Suppose $\mu$ and $\nu$ are two distributions over the same support $\Omega$; then, given one sample $s$ from the following distribution
	\begin{itemize}
    \item With probability $\rho$, sample $s$ from $\mu$;
    \item With probability $1-\rho$, sample $s$ from $\nu$;
    \end{itemize}
	The best probability we can decide whether $s$ came from $\mu$ or $\nu$ 
	is 
	\[
	\max(\rho, 1-\rho) + \min(\rho, 1-\rho)\cdot\tvd{\mu}{\nu}.
	\]
\end{fact}

\subsection{Information Theory Tools}
\label{subsec:info-theory}
We use information-theoretic tools in our proofs, and we include a review of the basic notions and properties used therein. We start with the definition of entropy. For a random variable $X$, we let $\HH(X)$ be the \emph{Shannon entropy} of $X$, defined as follows
\begin{definition}
\label{def:entropy}
Let $X$ be a discrete random variable with distributions $\mu_{X}$, the \emph{Shannon entropy} of $X$ is defined as
\begin{align*}
\HH(X) := \expect{\log(1/\mu(X))} =\sum_{x \in \text{supp}(X)} \mu(x)\cdot \log(\frac{1}{\mu(x)}),
\end{align*}
where $\text{supp}(X)$ is the support of $X$. If $X$ is a Bernoulli random variable, we use $H_2(p)$ to denote its Shannon entropy, where $P$ is the probability for $X=1$.
\end{definition}

We now give the definition of conditional entropy and mutual information.
\begin{definition}
\label{def:mutual-info}
Let $X$, $Y$ be two random variables, we define the \emph{conditional entropy} as \[\HH(X|Y)=\mathbb{E}_{y \sim Y}[\HH(X\mid Y=y)].\] 
With conditional entropy, we can define the \emph{mutual information} between $X$ and $Y$ as \[\II\paren{X;Y}:=\HH(X)-\HH(X\mid Y) = \HH(Y)-\HH(Y|X).\]
\end{definition}

The following information-theoretic facts (see e.g.~\cite{CoverT06} for details) are used in our lower bound proofs.
\begin{fact}
\label{fct:info-theory-facts}
Let $X$, $Y$, $Z$ be three discrete random variables:
\begin{itemize}
\item KL-divergence view of mutual information: $\II(X;Y) = \Exp_{y\sim Y}\bracket{\kl{X\mid Y=y}{X}}$.
\item $0\leq \HH(X)\leq \log(\card{\text{supp}(X)})$. In particular, if $X$ is a Bernoulli random variable, there is $H_2(p)\leq 1$.
\item $0\leq \II(X;Y)\leq \min\{\HH(X), \HH(Y)\}$.
\item Conditioning on independent random variable: let $X$ be independent of $Z$, then $\II(X;Y)\leq \II(X;Y\mid Z)$.
\item Chain rule of mutual information: $\II(X, Y; Z)=\II(X;Z)+\II(Y;Z\mid X)$.
\item Sub-additivity of entropy: $\HH(X,Y) \leq \HH(X) + \HH(Y)$, where $\HH(X, Y)$ is the joint entropy of variables $X, Y$.
\item Conditional independence of entropy: $\HH(X\mid Y,Z)=\HH(X\mid Y)$ if $X\perp Z\mid Y$, where the $\perp$ notation stands for independence.
\end{itemize}
\end{fact}

The following statement is known as the \emph{data processing inequality}, which says if $Y$ is obtained as a function of $X$, and $Z$ is obtained as a function of $Y$, then the mutual information between $X$ and $Z$ can only be lower than that between $X$ and $Y$.
\begin{proposition}
\label{prop:DPI}
Let $X$, $Y$, and $Z$ be random variables on finite supports, and we slightly abuse the notation to let $X,Y,Z$ to denote the distribution functions as well. Let $f$ be a deterministic function (no internal randomness), and suppose $Z=f(Y)$. Then, we have
\begin{align*}
	\II(X;Z)\leq \II(X;Y).
\end{align*}
\end{proposition}

The following statement characterizes the relationship between the ``zero mutual information'' and the independence of the conditional probability.

\begin{proposition}
\label{prop:mi-prob-indep}
Let $X$, $Y$, and $Z$ be random variables on finite supports, and suppose $\II(X;Y\mid Z=z)=0$. Then, for any realization $y\in Y$, there is
\begin{align*}
	\Pr(X\mid Z=z, Y=y) = \Pr(X\mid Z=z).
\end{align*}
\end{proposition}


\section{Missing proofs in \Cref{subsec:single-arm-complexity}}
\label{app:single-arm-proof}
We provide in this section the missing proofs of the hardness to ``identify'' and ``learn'' the distribution of a single arm as in \Cref{subsec:single-arm-complexity}. More restricted versions of the statement were already proved in work like \cite{MannorTs04} and \cite{AWneurips22}. We include the proofs for the versions used in this paper for completeness.

\subsection{Proof of \Cref{lem:arm-identify}}

\begin{proof}
Let $X$ be the random variable for the Bernoulli distribution with mean $\frac{1}{2}+\beta$ and $Y$ the random variable for the Bernoulli distribution with mean $\frac{1}{2}$. We use Fact~\ref{fact:distinguish-tvd} to argue that for a single arm pull, the probability for the algorithm to not identify the correct case is at least $\rho \cdot (1-\tvd{X}{Y})$. On the other hand, note that for the two Bernoulli distributions with means $\frac{1}{2}+\beta$ and $\frac{1}{2}$, there KL-divergence can be bounded as
\begin{align*}
\kl{X}{Y} &= (\frac{1}{2}+\beta) \cdot \log(1+2\beta) + (\frac{1}{2}-\beta)\cdot \log(1-2\beta)\\
&= \frac{1}{2}\cdot \log((1+2\beta)(1-2\beta)) + \beta\cdot \log(\frac{1+2\beta}{1-2\beta})\\
&\leq \beta\cdot \log(\frac{1+2\beta}{1-2\beta}) \tag{$\log(1-4\beta^{2})<0$}\\
&\leq \beta\cdot \log(2^{6\cdot\beta}) \tag{$\frac{1+2\beta}{1-2\beta}\leq 2^{6\cdot\beta}$ for $\beta\in(0,\frac{1}{6})$}\\
&= 6\cdot \beta^{2}.
\end{align*}

As such, using Pinsker's inequality (Fact~\ref{fact:pinsker}) that $\tvd{X}{Y}\leq \sqrt{\frac{1}{2}\cdot \kl{X}{Y}}$, and obtain that the probability for the algorithm to incorrectly identify the arm is at least $\rho\cdot (1-\sqrt{\frac{1}{2}\cdot \kl{X}{Y}})$. The bound can be generalized to $s$ samples: let $X^{[s]}$ and $Y^{[s]}$ be the distributions of $s$ samples from $X$ and $Y$. Then, we have:
\begin{align*}
\Pr\paren{\text{algorithm makes wrong prediction}} \geq \rho \cdot \paren{1-\sqrt{\frac{1}{2}\cdot \kl{X^{[s]}}{Y^{[s]}}}}.
\end{align*}
Using the fact that the samples are from independent and identical random variables, we can factorize $X^{[s]}$ with the marginal random variables of $\{X^{i}\}_{i=1}^{s}$ by the chain rule as follows:
\begin{align*}
\kl{X^{[s]}}{Y^{[s]})} &= \kl{X^{s}}{Y^{s}} + \kl{X^{[s-1]} \mid X^{s}}{Y^{[s-1]} \mid Y^{s}} \tag{by the chain rule}\\
&= \kl{X}{Y} + \kl{X^{[s-1]}}{Y^{[s-1]}} \tag{i.i.d. random variables}\\
&= \cdots\cdots\cdots\\
&= s\cdot \kl{X}{Y}.
\end{align*}
Therefore, combining the above steps, we have
\begin{align*}
\Pr\paren{\text{algorithm makes wrong prediction}} &\geq \rho\cdot\paren{1-\sqrt{\frac{1}{2}\cdot\kl{X^{s}}{y^{s}}}}\\
&\geq \rho\cdot\Big(1-\sqrt{\frac{1}{2}\cdot 6s\cdot \beta^{2}}\Big)\\
&\geq \rho \cdot \Big(1- 2\cdot\beta\cdot\sqrt{s}\Big).
\end{align*}
On the other hand, we want the error probability to be at most $\rho-\eps$, which means $\rho \cdot \Big(1- 2\beta\cdot\sqrt{s}\Big)\leq \rho-\eps$, which solves to $s\geq \frac{1}{4}\cdot \frac{\eps^4}{\rho^2 \beta^{2}}$.

\end{proof}

\subsection{Proof of \Cref{lem:arm-learn}}
\begin{proof}
Recall that $\sPi$ is the random variable for the transcript of the algorithm, and $\Theta$ is the random variable that controls from which case the instance is sampled.
We can write $\sPi = (\sPi_1,\sPi_2,\ldots,\sPi_s)$, where $\sPi_i$ denotes the random variable for the tuple of the $i$-th armed pull (recall that $\sPi$ and its realization $\sPi$ are defined as ordered tuples in \Cref{subsec:model}). We have, 
	\begin{align*}
		\mi{\Theta}{\sPi} &= \sum_{i=1}^{s} \mi{\Theta}{\sPi_i \mid \sPi^{<i}} \tag{by chain rule of mutual information} \\
		&= \sum_{i=1}^{s} \en{\sPi_i \mid \sPi^{<i}} - \en{\sPi_i \mid \Theta , \sPi^{<i}} \tag{by the definition of mutual information} \\
		&\leq \sum_{i=1}^{s} \en{\sPi_i } - \en{\sPi_i \mid \Theta , \sPi^{<i}} \tag{conditioning can only reduce the entropy} \\
		&= \sum_{i=1}^{s} \en{\sPi_i } - \en{\sPi_i \mid \Theta} \tag{because $\sPi_{i} \perp \sPi^{<i} \mid \Theta$ as knowing $\Theta$ fixes distribution of $\sPi_{i}$ to be either $\bern{1/2+\beta}$ or $\bern{1/2}$} \\
		&= \sum_{i=1}^{s} \mi{\Theta}{\sPi_i} \tag{by the definition of mutual information} \\
		&= \sum_{i=1}^{s} \Exp_{\theta \in \set{0,1}}\bracket{\kl{\sPi_i\mid \Theta=\theta}{\sPi_i}} \tag{by the connection of KL-divergence with mutual information} \\
		&= \sum_{i=1}^{s} \rho \cdot \kl{\sPi_i\mid \Theta=1 }{\sPi_i} + (1-\rho) \cdot \kl{\sPi_i \mid \Theta=0}{\sPi_i} \tag{by the distribution of $\theta$} \\
		&= \sum_{i=1}^{s} \rho \cdot \kl{\bern{\frac12+\beta}}{\bern{\frac12+\rho\cdot\beta}} + (1-\rho) \cdot  \kl{\bern{\frac12}}{\bern{\frac12+\rho\cdot\beta}} \tag{as distribution of $\sPi_i$ is $\rho \cdot \bern{\frac12 + \beta} + (1-\rho) \cdot \bern{\frac12} = \bern{\frac12 + \rho \cdot \beta}$} \\
		&\leq s \cdot \paren{\rho \cdot 6 \cdot (\rho \cdot \beta - \beta)^2 + 6 \cdot (1-\rho)\cdot (\rho \cdot \beta)^2} \tag{as proven in Lemma~\ref{lem:arm-identify}}  \\
		&\leq s \cdot \paren{12 \rho \cdot \beta^2} \leq \eps^3 \tag{by the upper bound on $s$}. 
	\end{align*}
The above calculation also implies that 
	\[
		\mi{\Theta}{\sPi} = \Exp_{\Pi} \bracket{\kl{\Theta}{\Theta \mid \sPi=\Pi}} \leq \eps^3. 
	\]
	By Markov bound, with probability $1-\eps$ over the choice of $\Pi \sim \sPi$, we have 
	\[
		\kl{\Theta}{\Theta \mid \sPi=\Pi} \leq \eps^2. 
	\]
	By Pinsker's inequality (Fact~\ref{fact:pinsker}), for any such $\Pi$, we have, 
	\[
		\tvd{\Theta}{\Theta \mid \sPi=\Pi} \leq \eps. 
	\]
	By the definition of total variation distance, this implies that 
	\[
		\card{\Pr\paren{\Theta=0}-\Pr\paren{\Theta=0 \mid \sPi=\Pi}} + \card{\Pr\paren{\Theta=1}-\Pr\paren{\Theta=1 \mid \sPi=\Pi}} \leq \eps. 
	\]
	By upper bounding each term separately and using the distribution of $\Theta$, we have, 
	\begin{align*}
		\card{\Pr\paren{\Theta=0 \mid \sPi=\Pi} - (1-\rho)} \leq \eps \quad \text{and} \quad \card{\Pr\paren{\Theta=1 \mid \sPi=\Pi} - \rho} \leq \eps,
	\end{align*}
	which concludes the proof. 
\end{proof}




\end{document}